\documentclass{article}
\usepackage[preprint,nonatbib]{neurips_2025} 

\usepackage[utf8]{inputenc}
\usepackage[T1]{fontenc}
\usepackage{microtype}
\usepackage{comment}
\usepackage{subcaption}
\usepackage{siunitx}

\usepackage{amsmath, amssymb, amsthm, mathtools}
\newtheorem{theorem}{Theorem}[section]

\newtheorem{remark}[theorem]{Remark}

\usepackage{graphicx}
\graphicspath{{figures/}}
\usepackage{booktabs}
\usepackage[hidelinks]{hyperref}
\usepackage[nameinlink]{cleveref}
\usepackage{tabularx} 

\usepackage[numbers,sort&compress]{natbib}

\title{Internet 3.0:  Architecture for a Web-of-Agents with it's Algorithm for Ranking Agents}
\author{
  Rajesh T. Krishanamachari\thanks{Equal contribution.} \\
  NYU \\
  \texttt{rtk267@nyu.edu}
  \And
  Srividya Rajesh\footnotemark[1] \\
  Independent Researcher \\
  \texttt{sk3957@caa.columbia.edu}
}

\begin{document}

\maketitle

\begin{abstract}
AI agents---powered by reasoning-capable large language models (LLMs) and integrated with tools, data, and web search---are poised to transform the internet into a \emph{Web of Agents}: a machine-native ecosystem where autonomous agents interact, collaborate, and execute tasks at scale. Realizing this vision requires \emph{Agent Ranking}---selecting agents not only by declared capabilities but by proven, recent performance. Unlike Web~1.0's PageRank, a global, transparent network of agent interactions does not exist; usage signals are fragmented and private, making ranking infeasible without coordination.

We propose \textbf{DOVIS}, a five-layer operational protocol (\emph{Discovery, Orchestration, Verification, Incentives, Semantics}) that enables the collection of minimal, privacy-preserving aggregates of usage and performance across the ecosystem. On this substrate, we implement \textbf{AgentRank-UC}, a dynamic, trust-aware algorithm that combines \emph{usage} (selection frequency) and \emph{competence} (outcome quality, cost, safety, latency) into a unified ranking. We present simulation results and theoretical guarantees on convergence, robustness, and Sybil resistance, demonstrating the viability of coordinated protocols and performance-aware ranking in enabling a scalable, trustworthy Agentic Web.
\end{abstract}


\newif\ifimpact


\section{Introduction}

The web is at an inflection point. AI-generated answers in search and the rapid adoption of conversational systems such as ChatGPT have sharply reduced traffic to individual websites---by 2025, over 60\% of Google searches ended without a click-through. This erosion of the open web's traffic undermines the incentive structures that have sustained it for decades and signals the need to move beyond a page-centric model. In parallel, AI agents---reasoning-capable LLM-based systems equipped with tools such as web search, APIs, and retrieval---are emerging as powerful intermediaries capable of executing complex tasks autonomously. Yet these agents remain constrained by the need to scrape and parse interfaces built for human--machine interaction, a mismatch akin to asking driverless cars to navigate streets designed exclusively for human drivers \cite{ehtesham2025survey, chang2025anpwhitepaper, Park2023,Li2023,Wu2023,Liu2024,Shinn2023,Wang2023}.

We term the emerging paradigm the \emph{Agentic Web}: a machine-native network in which autonomous AI agents are first-class citizens \cite{yang2025agentic}. In this vision, users convey intent to a personal agent, which plans, coordinates, and executes tasks by interacting with an internet-scale ecosystem of other agents. Websites, initially augmented and ultimately supplanted by agents, expose their capabilities through structured, discoverable interfaces. Most network traffic becomes agent-to-agent, and knowledge is both produced and consumed autonomously. Interaction shifts from browsing and reading to delegating and fulfilling---an evolution as fundamental as the transition from static web pages to platorms with dynamic, user-generated content.

Realizing this vision requires more than new agent capabilities; it demands a robust mechanism for selecting the best agents across a vast and dynamic ecosystem. We call this mechanism \textbf{AgentRank-UC}, a global ranking that combines \emph{usage} (how often an agent is chosen) and \emph{competence} (how well it performs) into a dynamic, trust-aware metric. Unlike Web~1.0's PageRank, which could be computed from publicly available hyperlinks \cite{BrinPage1998,Kleinberg1999,Haveliwala2002,Haveliwala2003,LangvilleMeyer2006,Gleich2015}, no global, transparent record of agent-to-agent interactions exists today. Without correctly-incentivized cooperation, usage and performance signals will likely remain private, fragmented, and siloed across platforms rendering large-scale, performance-aware discovery infeasible.

We address this gap with \textbf{DOVIS}, a five-layer operational protocol comprising \emph{Discovery, Orchestration, Verification, Incentives}, and \emph{Semantics}. Drawing upon previous research on layered internet architectures and observability systems \cite{Saltzer1984,Clark1988,ISO7498,RFC9000,Sigelman2010,OpenTelemetry}, DOVIS defines how usage and performance telemetry can be collected in a minimal, privacy-preserving, and verifiable way. \emph{Orchestration} specifies the caller-side reporting contract; \emph{Verification} binds reports to agent identities and enables auditing; \emph{Incentives} reward accurate reporting and penalize misrepresentation; \emph{Semantics} standardizes telemetry formats; and \emph{Discovery} consumes these signals to produce global rankings.

On top of this foundation, \textbf{AgentRank-UC} extends the link-analysis approach of PageRank to two evolving graphs: a usage graph capturing agent selection patterns and a competence graph capturing task outcomes. These graphs are integrated into a single ranking with recency decay, monotonicity guarantees, and robustness to Sybil attacks. Together, DOVIS and AgentRank-UC enable scalable, trustworthy, and economically viable coordination in the Agentic Web. We present simulation results and theoretical analysis demonstrating that coordinated telemetry and performance-aware ranking can make large-scale agent discovery both feasible and resilient in open, adversarial environments.

\section*{Contributions}

This paper makes the following contributions:

\begin{itemize}
    \item \textbf{Framing of the Agentic Web challenge.} We identify the fundamental obstacle to large-scale agent discovery: the absence of a transparent, global interaction graph comparable to hyperlinks on the Web. We argue that voluntary, privacy-preserving telemetry is necessary to enable competence-aware discovery.  

    \item \textbf{Design of the DOVIS protocol.} We introduce DOVIS, a five-layer protocol that specifies how agents publish minimal telemetry (OAT-Lite), how reports are orchestrated and verified, how incentives align participation, and how semantics ensure interoperability. This design provides the first concrete proposal for coordinated telemetry in open agent ecosystems.  

    \item \textbf{AgentRank-UC algorithm.} We propose AgentRank-UC, a minimal yet principled ranking algorithm that combines usage and competence into a single vector through coupled fixed-point equations. The algorithm incorporates recency, smoothing, priors for cold-starts, and guarantees existence and convergence.  

    \item \textbf{Theoretical analysis.} We show that AgentRank-UC satisfies desirable properties including uniqueness, monotonicity, and robustness to sparse or noisy telemetry, thereby ensuring well-posedness.  

    \item \textbf{Simulation framework and evaluation.} We design a simulation environment with archetypal agents, shocks, and adversarial behaviors. Results demonstrate that AgentRank-UC outperforms usage-only or competence-only baselines, adapts quickly to performance drift, and resists Sybil-style manipulation.  
\end{itemize}

Together, these contributions provide the first end-to-end framework for competence-aware discovery in the Agentic Web, establishing both the telemetry substrate (DOVIS) and the ranking algorithm (AgentRank-UC) required to make it operational.

\section{Protocol Layer}

\subsection{Minimal Telemetry Setting}

To ground our discussion, we begin with the simplest setting in which telemetry could make agent discovery possible. Imagine a future ecosystem of autonomous AI agents operating inside an \emph{agent marketplace}. Much like an app store for mobile software, the marketplace provides discovery and routing services so that callers---which may be either human users or other agents---can identify suitable callees that perform specific tasks. In contrast to the early web, however, there is no transparent hyperlink graph of agent interactions. Each invocation is private to the caller and callee, and unless some form of aggregate reporting is introduced, discovery algorithms cannot distinguish between agents that are merely popular and those that are genuinely competent.

We therefore propose a minimal form of telemetry: \emph{agents publish only aggregate statistics of their observed interactions}. This mechanism is deliberately lightweight, requiring no raw prompts, responses, or user data. Instead, it operates entirely on caller-side summaries that can be generated efficiently from local logs. 

Concretely, for each reporting epoch (for example, hourly or daily) and for each pair consisting of a callee $j$ and a task type $k$, a caller $i$ produces a compact record containing:
\begin{itemize}
    \item the total number of calls made to $j$ for task $k$, denoted \texttt{n\_calls};
    \item the number of those calls judged successful, \texttt{n\_success};
    \item cumulative measures of performance, including quality scores (\texttt{sum\_quality}), latency (\texttt{sum\_latency}), cost (\texttt{sum\_cost}), and risk or safety penalties (\texttt{sum\_risk});
    \item metadata fields identifying the reporting epoch and the caller’s identity, along with a digital signature to guarantee integrity.
\end{itemize}

We refer to this minimal schema as \emph{OAT-Lite} (Open Agent Telemetry - Lite). Although it records only a handful of aggregate fields, these statistics are precisely the sufficient inputs required by the AgentRank-UC algorithm described in Section~3. From the collection of such records, an indexer can reconstruct decayed counts $N_{ij}^{(k)}$, success totals $S_{ij}^{(k)}$, and averaged values for quality, latency, cost, and risk. With these in hand, it is straightforward to form the usage kernel $P$, the competence kernel $Q$, and ultimately a global ranking of agents.

This minimal telemetry setting provides an important conceptual baseline. Even if adoption were limited to bounded environments such as marketplaces, the publication of such simple aggregates would already enable competence-aware discovery at scale. More sophisticated mechanisms can then build on this foundation, but crucially, the minimal setting demonstrates that no invasive logging or heavyweight reporting is necessary for the Agentic Web to support ranking by usage and competence. This minimalist reporting approach also aims to be compatible with emerging agent interoperability efforts (e.g., MCP, ACP, A2A, ANP) \cite{ehtesham2025survey}.

\subsection{Orchestration: Who Reports, When, and How?}

In the near term, it is most natural to imagine deployment of DOVIS inside an \emph{agent marketplace}, a bounded ecosystem in which an indexer is already responsible for cataloguing available agents. Callers invoke callees for a given task type $k$, and each side of the interaction observes a different slice of information. The purpose of orchestration is to turn these private, asymmetric observations into minimal, consistent aggregates that the indexer can assemble into the sufficient statistics required by AgentRank-UC (Section~3), without exposing raw interaction content.  

\subsubsection*{Roles and responsibilities}  
The \emph{caller} is the primary reporter. For each callee--task pair $(j,k)$, the caller $i$ produces per-epoch aggregates: decayed call counts $N$, decayed successes $S$, and decayed sums for quality, latency, cost, and risk. These align naturally with what the caller can observe and map directly to the algorithm’s sufficient statistics.  

The \emph{callee} may optionally publish acknowledgments of the number of calls it has received in the same epoch, broken down by task type. These acknowledgments need not contain outcome fields, but they provide a lightweight cross-check of caller reports. This division of labor ensures that the caller remains the single source of truth for outcomes, avoiding duplication, while still allowing callees to strengthen integrity by confirming usage volumes.  Our choice to keep outcome semantics on the caller side, with callee acknowledgments, echoes end-to-end design principles and practical telemetry systems \cite{Saltzer1984,Sigelman2010,OpenTelemetry}.

\subsubsection*{Reporting cadence and windows}  
To avoid the inefficiency of continuous reporting, telemetry is submitted in fixed \emph{epochs} $E_t$, such as hourly or daily intervals in UTC. Within each epoch, callers aggregate their events locally; at epoch close, they submit one record for each $(i,j,k)$.  

The length of an epoch should be chosen relative to the exponential decay kernel $\omega(\tau) = e^{-\lambda \tau}$ used in AgentRank-UC, where the half-life is $H = \ln 2 / \lambda$. The requirement is that $|E_t| \ll H$, so that decayed aggregates remain fresh without requiring excessive submissions. Shorter epochs improve responsiveness to shocks, but increase overhead; the half-life parameter provides a principled way to justify the chosen cadence.  

\subsubsection*{Record structure and aggregation semantics}  
Each telemetry record is keyed by $(i,j,k,E_t)$ and contains:
\[
\{\texttt{n\_calls}, \texttt{n\_success}, \texttt{sum\_quality}, \texttt{sum\_latency}, \texttt{sum\_cost}, \texttt{sum\_risk}, \texttt{epoch\_id}, \texttt{caller\_id}, \texttt{signature}\}.
\]

These fields are defined as \emph{per-epoch snapshots} (not deltas) of the decayed sufficient statistics as of epoch close. If a record is resubmitted, the indexer retains the version with the most recent signature timestamp, ensuring idempotency. Using snapshots avoids double-counting under retries, while computing decay at the caller side keeps the payload compact.  

\subsubsection*{Transport, batching, and privacy}  
Callers transmit signed records to the marketplace indexer, either directly or via a relay. Relays can batch submissions, enforce rate limits, and provide privacy enhancements such as stripping network metadata or aggregating reports into $k$-anonymity buckets. Each record is compact: only one per nonzero $(j,k)$ per epoch, containing aggregate values and no raw prompts, responses, or user data. Direct submission is simplest, while relays offer scalability and privacy without changing semantics.  

\subsubsection*{Handling partial, missing, and late data}  
If a caller lacks information for some fields (for example, if cost or risk cannot be estimated), it may still submit call counts and successes. The indexer can then construct the usage kernel $P$, while the competence kernel $Q$ will automatically down-weight missing features by reverting to defaults in the utility function.  

The absence of a record for a given $(i,j,k)$ in an epoch is interpreted as zero mass. Records include an \texttt{epoch\_id}; the indexer accepts late submissions up to a bounded grace window (e.g., one or two epochs after close). After that window, late data are ignored or logged only for audit purposes. This tolerance maximizes participation while ensuring timely closure of epochs.  

\subsubsection*{Assembly at the indexer}  
At the end of each epoch, the indexer:
\begin{enumerate}
    \item Deduplicates records by key $(i,j,k,E_t)$, using last-write-wins.  
    \item Validates signatures and checks field ranges and units.  
    \item Computes usage weights $U_{ij} = \sum_k N_{ij}^{(k)}$ and competence weights $C_{ij} = \sum_k N_{ij}^{(k)} \phi(u_{ij}^{(k)})$, exactly as required in Section~3.  
    \item Row-normalizes these into the stochastic kernels $P$ and $Q$, backing off empty rows to priors $v$ and $w$.  
    \item Runs the fixed-point iterations for usage and competence ranks ($x$ and $y$), then combines them into the final ranking vector $r$.  
\end{enumerate}

Thus, orchestration ends at the point of delivering consistent aggregates to the indexer. Verification, incentives, and semantics---addressed in the next subsections---provide additional guarantees on integrity, participation, and interpretability.  

\subsubsection*{Default operating point}  
For practical deployment, we recommend: hourly epochs within marketplaces, and daily epochs in open federations; decay half-life $H$ between one and seven days, balancing recency and stability; snapshot submissions by callers, with optional callee acknowledgments; idempotent retries permitted, with last-write-wins resolution; late submissions accepted for up to two epochs; and sparse reporting, with only keys where \texttt{n\_calls} $>0$ included.  

These defaults provide a robust baseline configuration, while leaving room for more stringent or relaxed policies in different ecosystems.

\subsection{Verification: Ensuring Integrity of Telemetry}

The minimal telemetry scheme presumes that agents report honestly. In reality, however, strategic misreporting poses the most serious threat. Without safeguards, rankings could be manipulated by inflating successes, suppressing failures, fabricating calls, or cycling through identities. Verification is therefore essential to ensure that the aggregated records ingested by the indexer remain credible enough for AgentRank-UC to produce meaningful results.  

\subsubsection*{Core threats}  
Several distinct attack surfaces must be considered. Our exposition below on the verification surface is shaped by long-standing lessons from internet protocol design \cite{Clark1988,RFC9000} and modern secure computation/attestation techniques for audits without raw data \cite{Bonawitz2017,Pinkas2014,DeCristofaro2012,Costan2016}. A caller may practice \emph{success inflation}, reporting more successful outcomes than were actually observed, thereby artificially boosting a callee’s competence weight. Conversely, \emph{failure suppression} hides errors or harmful outcomes, biasing averages upward. Groups of colluding agents may engage in \emph{usage pumping}, repeatedly calling one another to fabricate high call volumes without delivering real utility. Finally, agents penalized for dishonesty may attempt \emph{identity churn}, re-entering under new identifiers and presenting themselves as ``newcomers.''

\subsubsection*{Lightweight defenses}  
To mitigate these threats without undermining the minimality of the protocol, DOVIS introduces three complementary layers of defense.  

First, each telemetry record is bound to a persistent identity through \emph{cryptographic signatures}. Callers sign their reports with registered keys; these keys are recorded with the marketplace indexer and made visible for auditing. While signatures cannot guarantee truthful content, they prevent forgery and impersonation, ensuring that every record can be attributed to a stable identity.  

Second, \emph{callee acknowledgments} provide a lightweight cross-check. Although callees cannot observe success or quality, they do know the number of invocations received. By publishing per-epoch counts of calls, they allow the indexer to compare these totals against caller-reported \texttt{n\_calls}. Significant discrepancies beyond a tolerance threshold can trigger further investigation. This mechanism is optional but inexpensive, and it raises overall integrity without duplicating semantics.  

Third, the protocol incorporates \emph{randomized audits}. In each epoch, the indexer samples a small fraction (e.g., 1--5\%) of caller--callee--task triples for deeper scrutiny. Audits may involve comparing caller reports against callee acknowledgments, replaying logged call identifiers if hashed IDs are available, or requesting attested proof from trusted execution environments. Even with a low audit probability, the prospect of penalties makes systematic cheating unattractive.  

\subsubsection*{Weighting and penalties}  
Verification also relies on adaptive weighting. Reports from stronger identities, such as verified accounts, staked agents, or runtimes attested via Trusted Execution Environments, are given higher prior weight in AgentRank-UC. Unverified agents continue to contribute, but their influence is proportionally reduced. When audits reveal serious misreporting, penalties are imposed: future contributions are down-weighted or excluded, and visibility in rankings is diminished. At the same time, the protocol tolerates minor discrepancies within a small $\varepsilon$-tolerance, recognizing that benign reporting noise should not be punished.  

\subsubsection*{Handling Sybils and churn}  
Two additional risks require attention. \emph{Sybil detection} targets coordinated usage-pumping rings, which typically manifest as tightly connected clusters with sudden spikes in mutual calls. Indexers can monitor graph structure to detect such anomalies and trigger audits. To counter \emph{identity churn}, the protocol applies ramp-up rules: new identifiers begin with limited influence, which increases gradually over several epochs unless supported by stronger verification such as staking or attestation.  

\subsubsection*{Practical baseline}  
A practical baseline configuration of DOVIS verification includes the following: cryptographic signatures are mandatory for all reports; callee acknowledgments are optional but recommended; the audit regime samples roughly 1--5\% of edges per epoch with penalties for dishonesty; and all agents are required to maintain persistent keys, with higher trust assigned to attested or staked identities.  

\medskip
\noindent\textbf{Summary.} Verification complements orchestration by binding reports to persistent identities, introducing optional cross-checks, and incorporating probabilistic audits backed by penalties. These measures deter manipulation while keeping overhead low, ensuring that the telemetry stream remains credible enough to support competence-aware discovery through AgentRank-UC.

\subsection{Incentives: Why Would Agents Contribute?}

The minimal telemetry scheme implicitly assumes that agents willingly generate and publish reports. In practice, however, this is not guaranteed. Producing telemetry, even in aggregate form, imposes costs in computation, bandwidth, and time, while the benefits of improved discovery are distributed across the ecosystem. This creates the classic conditions for free-riding and incentive misalignment. Without explicit mechanisms to reward participation and discourage dishonesty, reporting rates may remain too low for AgentRank-UC to operate reliably.  

\subsubsection*{Core challenges}  
Several incentive problems are particularly salient.  

First, there is the issue of \emph{free-riding}: an agent can benefit from improved rankings powered by the telemetry of others, while withholding its own reports. Second, even minimal aggregation entails a \emph{reporting cost}; agents operating under tight computational or bandwidth constraints may choose to save resources by not reporting. Third, there is the \emph{risk of honesty}: if agents believe competitors are misreporting, they may feel disadvantaged by submitting accurate telemetry, which could erode trust in the scheme. Finally, new entrants face a \emph{cold-start disincentive}. If agents assume that reporting has little effect before they have accumulated significant history, they may decline to participate in the early stages when adoption is most critical.  

\subsubsection*{Candidate incentive mechanisms}  
To address these challenges, DOVIS supports a range of incentive mechanisms, each designed to align individual interests with the collective good.  

One approach is to provide \emph{exposure rewards}. Indexers can boost the visibility of agents who consistently provide telemetry. For instance, when two agents are otherwise tied in competence, the agent with more reliable reporting is ranked slightly higher. This creates a virtuous cycle: honest reporters attract more calls, generating richer telemetry that further strengthens their rank.  

A complementary mechanism is to impose \emph{penalties for dishonesty}. As described in Section~2.3, failed audits trigger slashing penalties, such as down-weighting future contributions, reducing exposure in rankings, or even temporary suspension from the marketplace. The key idea is that the expected cost of lying must outweigh any short-term gains.  

Another option is to introduce \emph{telemetry credits}. In this scheme, submitting valid reports earns credits redeemable for tangible benefits such as reduced platform fees, preferential placement in search results, or access to premium services. This reframes telemetry not as an overhead, but as a resource-generating action.  

Finally, to encourage new participants, DOVIS incorporates \emph{cold-start support}. AgentRank-UC already assigns non-zero baseline ranks through priors (Section~3), ensuring that newcomers are not invisible by default. Marketplaces can make this explicit: ``Your rank starts visible; consistent telemetry will boost you faster.'' This messaging lowers the barrier to entry and ensures that honest reporting is valuable from the outset.  

\subsubsection*{Practical baseline}  
A practical baseline configuration would include: a modest exposure boost for telemetry reporters; rank down-weighting on failed audits as a default penalty; optional credits or fee discounts as marketplace-specific extensions; and cold-start priors as a built-in guarantee that participation is never futile at entry.  

\medskip
\noindent\textbf{Summary.} Incentives close the loop between the individual costs of producing telemetry and the collective benefits of improved discovery. By combining exposure rewards, penalties for dishonesty, optional credit systems, and cold-start support, DOVIS ensures that voluntary reporting becomes not merely a cost but a net positive for participating agents.  Our incentive mechanisms for honest telemetry dovetail with federated and privacy-preserving learning ecosystems where participation costs must be offset \cite{McMahan2017,DworkRoth2014,Abadi2016,Bonawitz2017}.

\subsection{Semantics: What Do the Numbers Mean?}

Even if orchestration, verification, and incentives are successfully addressed, telemetry remains useless unless participants interpret the fields they report in a consistent way. In the minimal scheme, ambiguity about units, scales, or task labels would make comparisons across agents misleading. For rankings to be meaningful, semantics must ensure that ``success,'' ``latency,'' or ``translation'' carry the same operational meaning across the ecosystem.  

\subsubsection*{Core challenges}  
Several risks arise when semantics are left undefined. A common issue is \emph{metric inconsistency}: one agent might log latency in seconds while another records milliseconds; quality could be represented on a five-point ordinal scale or normalized to the unit interval $[0,1]$. A second concern is \emph{task ambiguity}. Without a shared taxonomy, an invocation labeled as ``translation'' could mean human-language translation for one agent and chemical formula conversion for another.  

A further challenge is \emph{schema drift}. As new metrics are introduced---such as energy consumption or fairness indicators---older reports may no longer align with the current schema. Finally, \emph{interoperability} is at stake: agents from different marketplaces may report superficially similar fields that differ in subtle but consequential ways, impeding federation across ecosystems.  In particular, interoperability with existing and emerging agent network specifications like ANP must be a key goal for OAT-Lite/OAT-Full evolution \cite{chang2025anpwhitepaper,ehtesham2025survey}.

\subsubsection*{Candidate solutions}  
DOVIS addresses these issues by defining a minimal but explicit semantic layer.  

At its core is the \emph{OAT-Lite baseline schema}. Every report must include fixed fields---\texttt{n\_calls}, \texttt{n\_success}, \texttt{sum\_quality}, \texttt{sum\_latency}, \texttt{sum\_cost}, \texttt{sum\_risk}---together with metadata such as \texttt{epoch\_id}, \texttt{caller\_id}, and \texttt{signature}. This ensures that all reports contain the sufficient statistics required for AgentRank-UC (Section~3).  

Second, DOVIS mandates \emph{standardized units and scales}. Latency is measured in milliseconds; cost is expressed in normalized credits; and quality and risk are reported on the unit interval $[0,1]$. Each report also carries a \texttt{schema\_version} field to document which unit conventions apply. These requirements guarantee that indexers can safely aggregate and compare telemetry across agents.  

Third, a \emph{task taxonomy registry} provides unique identifiers for common task types (e.g., \texttt{task\_42} = ``translation: human languages''). New tasks can be added as the ecosystem evolves, but each must be versioned and documented. This prevents ambiguous or overloaded labels and ensures that per-task rankings remain interpretable.  

Finally, DOVIS supports \emph{schema versioning and backward compatibility}. Each record explicitly carries its schema version, and indexers implement normalization layers so that older reports remain usable. For example, if earlier records reported quality on a 1--5 scale, they can be rescaled to $[0,1]$ before being integrated with newer reports.  

\subsubsection*{Practical baseline}  
A practical baseline for semantics would require the OAT-Lite schema with canonical units and a schema version field as mandatory elements; the use of global task registry identifiers for per-task reports as strongly recommended; and optional richer metadata---such as language pairs or domains---in an extended OAT-Full specification.  

\medskip
\noindent\textbf{Summary.} Semantics ensure that minimal aggregates actually line up across agents. By standardizing schema fields, defining canonical units, and introducing versioned task identifiers, DOVIS guarantees that telemetry is interpretable in the present and robust to evolution in the future.

\subsection{The DOVIS Protocol}

The DOVIS protocol defines how agents in a marketplace (or open ecosystem) generate, exchange, and verify telemetry records so that discovery algorithms such as AgentRank-UC (Section~3) can be computed reliably. DOVIS has a layered architecture, with each layer addressing a specific operational challenge:

\begin{itemize}
    \item \textbf{Discovery.} Indexers aggregate telemetry reports and compute rankings.  
    \item \textbf{Orchestration.} Callers aggregate and emit telemetry at the close of each epoch, while callees may optionally publish acknowledgments.  
    \item \textbf{Verification.} Reports are digitally signed, can be cross-checked against acknowledgments, and are subject to probabilistic audits.  
    \item \textbf{Incentives.} Reporters are rewarded with exposure or credits, while misreporters face penalties and rank suppression.  
    \item \textbf{Semantics.} Telemetry fields, units, and task taxonomies are standardized through the OAT-Lite schema.  
\end{itemize}

Together, these layers form a minimal but extensible substrate for performance-aware discovery in the Agentic Web.  Our layering mirrors classical reference models and modern transport/telemetry stacks \cite{ISO7498,Clark1988,RFC9000,OpenTelemetry}.

\subsubsection*{Message types}  
DOVIS defines three minimal record types:  

\begin{itemize}
    \item \textbf{Caller Report (mandatory).}  
    A caller submits one record per callee--task pair in each epoch:
    \begin{verbatim}
    { epoch_id, caller_id, callee_id, task_id,
      n_calls, n_success,
      sum_quality, sum_latency, sum_cost, sum_risk,
      schema_version, signature }
    \end{verbatim}
    This record represents the caller-side aggregate snapshot of outcomes for the specified epoch and task type.  

    \item \textbf{Callee Acknowledgment (optional).}  
    A callee may submit an acknowledgment record of calls received:
    \begin{verbatim}
    { epoch_id, callee_id, task_id,
      n_calls_received,
      schema_version, signature }
    \end{verbatim}
    This record confirms the total volume of calls received, without reporting outcomes.  

    \item \textbf{Audit Request/Response (optional, indexer-initiated).}  
    Indexers may issue a request challenging a caller or callee to provide evidence, such as hashed call identifiers. The challenged agent responds with proof, and mismatches are penalized.  
\end{itemize}

\subsubsection*{Orchestration rules}  
\begin{itemize}
    \item \textbf{Epoch partitioning.} Time is divided into epochs (e.g., hourly or daily). Callers must submit exactly one Caller Report per nonzero $(callee, task)$ pair per epoch.  
    \item \textbf{Decayed aggregation.} Caller Reports must reflect exponentially decayed statistics, computed using the half-life parameter $H$, as of the close of the epoch.  
    \item \textbf{Idempotency.} If multiple submissions are made for the same $(caller, callee, task, epoch)$, the indexer retains only the latest record, ensuring last-write-wins semantics.  
    \item \textbf{Relays.} Reports may pass through relays for batching, anonymization, or privacy; such relays must preserve semantics.  
    \item \textbf{Grace window.} Indexers accept late submissions for up to two epochs beyond the reporting interval, after which reports are discarded.  
\end{itemize}

\subsubsection*{Verification rules}  
\begin{itemize}
    \item \textbf{Digital signatures.} All records must be signed using persistent keys, and indexers must verify signatures before ingestion.  
    \item \textbf{Cross-checking.} Caller Reports for \texttt{n\_calls} may be compared with optional Callee Acknowledgments; significant discrepancies trigger an audit.  
    \item \textbf{Audit regime.} In each epoch, the indexer randomly samples approximately 1--5\% of edges for deep audit, which may involve replaying hashed call identifiers or comparing caller and callee totals.  
    \item \textbf{Identity weighting.} Reports from stronger identities (such as staked, verified, or TEE-attested agents) carry greater weight in rankings, while unverifiable reporters are down-weighted.  
    \item \textbf{Penalties.} Failed audits result in penalties, including rank weight reduction, visibility suppression, or suspension.  
\end{itemize}

\subsubsection*{Incentive mechanisms}  
\begin{itemize}
    \item \textbf{Exposure rewards.} Agents who consistently report telemetry may receive modest boosts in visibility within discovery results.  
    \item \textbf{Penalties for dishonesty.} Misreporting detected through audits leads to slashing, including down-weighting of contributions, rank suppression, or exclusion from the marketplace.  
    \item \textbf{Telemetry credits (optional).} Marketplaces may implement a credit system in which valid telemetry submissions earn credits redeemable for reduced fees, higher placement, or other benefits.  
    \item \textbf{Cold-start fairness.} AgentRank-UC priors guarantee that new agents remain visible even before telemetry has accumulated, ensuring that early reporting is not futile.  
\end{itemize}

\subsubsection*{Semantics: OAT-Lite schema}  
\begin{itemize}
    \item \textbf{Mandatory fields.} All records must include the fields \texttt{n\_calls}, \texttt{n\_success}, \texttt{sum\_quality}, \texttt{sum\_latency}, \texttt{sum\_cost}, and \texttt{sum\_risk}.  
    \item \textbf{Canonical units.} Latency must be measured in milliseconds, cost expressed in normalized credits, and quality and risk normalized to the interval $[0,1]$.  
    \item \textbf{Task identifiers.} Each record must include a task identifier drawn from the global task taxonomy registry.  
    \item \textbf{Schema versioning.} Each record must carry a \texttt{schema\_version} field, and indexers must maintain backward compatibility across versions.  
\end{itemize}

\subsubsection*{Indexer workflow}  
At the close of each epoch, the indexer executes the following sequence:  

\begin{enumerate}
    \item It ingests Caller Reports (mandatory) and Callee Acknowledgments (optional).  
    \item It deduplicates records by the key $(caller, callee, task, epoch)$ using last-write-wins.  
    \item It verifies all signatures and evaluates audit responses.  
    \item It constructs usage and competence weights $U_{ij}$ and $C_{ij}$.  
    \item It normalizes these weights into row-stochastic kernels $P$ and $Q$, applying priors for empty rows.  
    \item It runs the AgentRank-UC algorithm to obtain the rank vector $r$.  
    \item It publishes the resulting rankings back to the ecosystem to support discovery.  
\end{enumerate}

\medskip
\noindent\textbf{Summary.} The DOVIS protocol specifies message formats, orchestration rules, verification policies, incentive mechanisms, semantic standards, and indexer workflow. Together, these rules form a minimal but concrete specification that can be implemented in practice, while remaining extensible as agent ecosystems mature.

\subsection{Deployment and Extensibility}

Although the DOVIS protocol is designed with internet-scale agent ecosystems in mind, its practical deployment will likely begin in more controlled environments. The most natural starting point is within \emph{agent marketplaces}, where incentives and verification are already aligned: the marketplace operator has both the means and the motivation to enforce reporting rules, run audits, and provide exposure rewards. Within these bounded settings, the costs of orchestration are manageable, and the benefits of telemetry are immediately visible in the quality of discovery services.  

Looking further ahead, \emph{federated deployment} across multiple marketplaces becomes possible once basic semantic agreements are established. Shared task registries and adherence to the OAT-Lite schema would allow telemetry to be combined across ecosystems, producing rankings that extend beyond the boundaries of any single platform. This step is more challenging, since no single operator controls incentives or verification globally, but the trajectory resembles the history of internet protocols: islands of adoption first, federation later.  Federated deployment across marketplaces will require convergence of schemas and registries, similar in spirit to efforts documented in recent agent-protocol surveys and specifications \cite{ehtesham2025survey,chang2025anpwhitepaper}.

Finally, DOVIS is deliberately designed to be \emph{minimal and extensible}. The OAT-Lite schema provides only the fields required for AgentRank-UC, but nothing prevents richer extensions. An ``OAT-Full'' profile could incorporate additional metrics such as energy consumption, fairness, or interpretability; stronger privacy mechanisms such as secure aggregation or differential privacy; and more sophisticated incentive mechanisms, such as tokenized credits or stake-based rewards. By layering these extensions on top of the minimal core, DOVIS can evolve as agent ecosystems mature, while remaining interoperable with the simple baseline defined here.  

\medskip
\noindent With orchestration, verification, incentives, and semantics in place, DOVIS establishes the minimal telemetry substrate required for performance-aware discovery. The next question is how to turn this stream of aggregated reports into a principled, competence-aware ranking of agents. In Section~3, we present \emph{AgentRank-UC}, an algorithm that combines usage and competence signals into a single ranking vector, ensuring that discovery reflects not only popularity but also demonstrated quality.

\section{Agent Rank Algorithm}

\subsection{Background \& Related Work}
PageRank \cite{BrinPage1998} introduced the now-classic idea of ranking entities by propagating importance through a graph. While originally applied to web pages and hyperlinks, the underlying principle — importance flows from well-connected, high-quality nodes — has since inspired numerous extensions, including TrustRank (trust propagation), Personalized PageRank (context biasing), and Weighted PageRank (metadata-aware link weights) \cite{Kleinberg1999,Haveliwala2002,Haveliwala2003,Gyongyi2004,LangvilleMeyer2006,Gleich2015}.

Modern search systems no longer rely on PageRank in its original form, instead combining many signals — semantic relevance, quality, freshness, and trust — into machine-learned ranking models. Nonetheless, PageRank remains a useful conceptual foundation for importance propagation in graphs.

In our setting, the same principle applies to agents in the Agentic Web. AgentRank-UC extends this idea by incorporating usage (how often an agent is called) and competence (success and quality of outcomes) into a unified ranking model, making discovery both capability-aware and performance-driven. The fusion can be seen as a multiplicative analogue of combining authority and popularity, akin in spirit to prior propagation models but explicitly competence-aware \cite{Kleinberg1999,Gyongyi2004,Gleich2015}.

\subsection{Minimal Problem Model and Essential Requirements}

We consider a setting with a finite set of agents $\mathcal{A} = \{1, \dots, n\}$ that operate on tasks belonging to a set of types $\mathcal{T}$. The primary data source for ranking is an interaction log that records calls between agents. Each record corresponds to an invocation from agent $i$ to agent $j$ on a task of type $k \in \mathcal{T}$ at time $t$, along with outcome and performance metadata. Specifically, the record contains: a binary success indicator $z \in \{0,1\}$; a quality score $q \in [0,1]$ reflecting graded performance; a latency $\ell > 0$ in time units; a cost $c \ge 0$ representing resource consumption; and a risk or safety score $r \in [0,1]$ capturing the likelihood or severity of harmful outcomes.

To ensure that more recent evidence is weighted more heavily than stale observations, we introduce a recency kernel $\omega(t) = e^{-\lambda (T - t)}$, where $T$ denotes the current time and $\lambda > 0$ is a decay parameter. Using this weighting, we can compute time-decayed aggregates such as the total number of calls from $i$ to $j$ of type $k$, denoted $N_{ij}^{(k)} = \sum \omega(t)$ over relevant records, and time-decayed estimates of success probabilities $\widehat{p}_{ij}^{(k)}$ using smoothed estimators such as Beta--Bernoulli posteriors. Similar aggregates are maintained for latency, cost, risk, and quality to support a multi-faceted assessment of agent performance.

An effective agent ranking algorithm must first integrate two fundamentally different kinds of evidence: how frequently an agent is relied upon by others (usage or influence) and how well the agent performs when called upon (competence or quality). Both aspects are essential; usage alone may overvalue popular but unreliable agents, while competence alone may undervalue highly capable agents that are not yet widely known.

Second, the ranking formulation must be mathematically well-posed, with a unique solution that can be computed efficiently at scale. This typically requires constructing the algorithm so that it can be expressed as a contraction mapping, enabling convergence to a fixed point through iterative methods.

Third, the algorithm must satisfy certain operational constraints to ensure robustness and fairness. In particular, it should incorporate recency so that current performance influences rankings more strongly than outdated behavior; it should be monotonic in the sense that improving an agent’s observed outcomes cannot reduce its score; and it should accommodate cold-start agents through the use of informative priors, preventing them from being assigned negligible scores solely due to lack of historical data.

To meet these requirements, it is natural to represent the agent ecosystem as two distinct weighted directed graphs. The first encodes \emph{usage patterns}, where the weight on an edge from $i$ to $j$ reflects how often $i$ invokes $j$, adjusted for recency. The second encodes \emph{competence patterns}, where edge weights are modulated by the observed success rates and quality metrics associated with those invocations, again with recency adjustments and possible penalties for high latency, high cost, or safety concerns.

From these graphs we derive two row-stochastic matrices:
\begin{itemize}
    \item $P$, the \emph{usage kernel}, whose entry $P_{ij}$ represents the normalized propensity of calls from agent $i$ to agent $j$ in the usage graph.
    \item $Q$, the \emph{competence kernel}, whose entry $Q_{ij}$ represents the normalized strength of competence signals from agent $i$ to agent $j$ in the competence graph.
\end{itemize}

In addition, we define two probability vectors over agents:
\begin{itemize}
    \item $v$, the \emph{usage prior}, which distributes a baseline amount of usage ``mass'' independent of observed interactions. This prevents agents with few or no observations from being ranked at zero and can be uniform or biased toward agents with trusted external credentials.
    \item $w$, the \emph{competence prior}, which distributes a baseline amount of competence ``mass'' independent of observed usage, possibly informed by offline testing or certification.
\end{itemize}

The separation into $P$ and $Q$ allows the algorithm to propagate influence and quality independently before combining them. The inclusion of priors $v$ and $w$ ensures that rankings remain well-defined for all agents and provides a mechanism for incorporating external knowledge into the ranking process.

The ranking process is formulated as two coupled fixed-point problems. The \emph{usage rank} vector $x$ satisfies
\[
x = \alpha P^\top x + (1 - \alpha)v,
\]
which can be interpreted as a random walk where, with probability $\alpha$, a ``usage signal'' follows observed interaction patterns via $P$, and with probability $1 - \alpha$, it ``teleports'' to a baseline distribution $v$.

Similarly, the \emph{competence rank} vector $y$ satisfies
\[
y = \beta Q^\top y + (1 - \beta)w,
\]
meaning that, with probability $\beta$, competence propagates along observed competence-weighted edges in $Q$, and with probability $1 - \beta$, it resets to the baseline $w$.

The final \emph{AgentRank-UC} score combines usage and competence into a single measure:
\[
r = \mathrm{normalize}\left( x^{p} \odot y^{\,1-p} \right), \quad p \in [0,1],
\]
where $p$ controls the trade-off between usage and competence. This formulation ensures that both forms of evidence are incorporated, that the system has a unique stationary solution under mild conditions, and that the resulting ranking is responsive to current performance while robust to sparse data.

\subsection{AgentRank-UC: Minimal Algorithm}

\textbf{Inputs and outputs.} The algorithm takes as input an interaction log of records $(i \!\to\! j, k, t; z, q, \ell, c, r)$ as defined in the problem model; hyperparameters $\lambda>0$ (time decay), $\alpha_0,\beta_0>0$ (success smoothing), $\alpha,\beta \in (0,1)$ (teleport weights), balance $p \in [0,1]$, and priors $v,w \in \Delta^{n-1}$ (uniform or informed). The output is a rank vector $r \in \Delta^{n-1}$ over agents; optionally per-task ranks $r^{(k)}$.

\paragraph{Step 1 --- Recency weighting and aggregation.}  
For each ordered pair $(i,j)$ and task type $k$, compute the decayed count 
\[
N_{ij}^{(k)} = \sum \omega(t) \quad \text{with} \quad \omega(t) = e^{-\lambda (T - t)},
\]
the decayed successes $S_{ij}^{(k)} = \sum \omega(t) z$, and decayed means $\bar{q}_{ij}^{(k)}, \bar{\ell}_{ij}^{(k)}, \bar{c}_{ij}^{(k)}, \bar{r}_{ij}^{(k)}$.  
\emph{Motivation.} Exponential decay ensures that current performance influences the model more than stale behavior while allowing a single linear pass over the log to build sufficient statistics.

\paragraph{Step 2 --- Smoothed success posteriors.}  
Estimate success on $(i,j,k)$ with a Beta--Bernoulli posterior
\[
\widehat{p}_{ij}^{(k)} = \frac{\alpha_0 + S_{ij}^{(k)}}{\alpha_0 + \beta_0 + N_{ij}^{(k)}}.
\]
\emph{Motivation.} Conjugate smoothing produces calibrated probabilities in sparse regimes and guarantees monotonic improvements as observed successes increase.

\paragraph{Step 3 --- Build usage and competence edge weights.}  
Define usage weights $U_{ij} = \sum_k N_{ij}^{(k)}$. For competence, compute a task-level utility
\[
u_{ij}^{(k)} = \theta_1 \,\mathrm{logit}(\widehat{p}_{ij}^{(k)}) - \theta_2 \log(1+\bar{\ell}_{ij}^{(k)}) - \theta_3 \log(1+\bar{c}_{ij}^{(k)}) - \theta_4 \,\bar{r}_{ij}^{(k)} + \theta_5 \,\bar{q}_{ij}^{(k)},
\]
and aggregate via 
\[
C_{ij} = \sum_k N_{ij}^{(k)} \,\phi\big(u_{ij}^{(k)}\big), \quad \phi(x) = \log(1 + e^x) \ \text{(softplus)}.
\]
\emph{Motivation.} The usage weights capture who relies on whom after recency adjustment, while the competence weights reward successful, high-quality outcomes and softly penalize cost, latency, and risk, with softplus ensuring positivity, smoothness, and diminishing returns so that volume alone cannot dominate.

\paragraph{Step 4 --- Row-normalize to kernels $P$ and $Q$.}  
For each source $i$, set $P_{ij} = U_{ij}/\sum_{j'} U_{ij'}$ if the row has mass and otherwise $P_{ij} = v_j$; analogously set $Q_{ij} = C_{ij}/\sum_{j'} C_{ij'}$ or $Q_{ij} = w_j$.  
\emph{Motivation.} Row normalization produces valid Markov kernels while the backoff to priors prevents dangling or empty rows from breaking the stochastic process.

\paragraph{Step 5 --- Choose priors $v$ and $w$.}  
Let $v$ be a baseline usage distribution (uniform or biased toward verified/attested agents) and $w$ be a baseline competence distribution (uniform or informed by external benchmarks or certifications).  
\emph{Motivation.} Priors provide cold-start fairness, encode external evidence, and ensure ergodicity of the subsequent fixed-point iterations.

\paragraph{Step 6 --- Solve usage and competence fixed points.}  
Iterate
\[
x^{(t+1)} = \alpha P^\top x^{(t)} + (1 - \alpha)v, \qquad
y^{(t+1)} = \beta Q^\top y^{(t)} + (1 - \beta)w
\]
from $x^{(0)} = v$ and $y^{(0)} = w$ until $\|x^{(t+1)} - x^{(t)}\|_1$ and $\|y^{(t+1)} - y^{(t)}\|_1$ fall below a tolerance.  
\emph{Motivation.} Teleportation yields contraction mappings with unique fixed points and linear convergence, while separating usage and competence preserves interpretability and control.

\paragraph{Step 7 --- Combine usage and competence into AgentRank-UC.}  
Set 
\[
r = \mathrm{normalize}\big(x^{\,p} \odot y^{\,1-p}\big), \quad p \in [0,1].
\]
\emph{Motivation.} The geometric fusion penalizes imbalance (popular-but-poor or excellent-but-undiscovered), is scale-invariant, and exposes a single, interpretable trade-off parameter.

\medskip
\noindent\textbf{Design, complexity, and implementation.} Modeling usage and competence with separate kernels prevents conflation of popularity and performance, Beta--Bernoulli smoothing stabilizes sparse edges, softplus aggregation limits dominance by raw volume, and teleportation with priors guarantees existence and uniqueness. Building the sufficient statistics is linear in the number of log events, each power iteration is proportional to the number of nonzero edges in $P$ and $Q$, and a modest number of iterations typically suffices. In practice, apply small numerical floors when normalizing rows, cap heavy-tailed latency/cost via $\log(1+x)$, and, if desired, compute per-task ranks by repeating the procedure on task-filtered logs and aggregating them for a global score.

\subsection{Deployment and Use of AgentRank-UC}

AgentRank-UC is designed to operate at the level of a marketplace or ecosystem indexer. At the close of each epoch, the indexer ingests the OAT-Lite telemetry described in Section~2, constructs the usage and competence kernels $P$ and $Q$, and computes fixed-point solutions for the usage vector $x$ and competence vector $y$. The combined score
\[
r^{(k)} = \mathrm{normalize}\!\left(x^{\,p} \odot y^{\,1-p}\right)
\]
is produced separately for each task type $k$, yielding a set of per-task rankings that are disseminated to all participating agents.  

The rank vectors are not merely descriptive. They serve as inputs to the \emph{selection rule} that callers apply when choosing among potential callees. In practice, a caller sampling a callee $j$ for task $k$ selects according to a probability of the form
\[
\Pr(i \to j \mid k) \;\propto\; (1-\rho-\gamma)\,\text{popularity}_j \;+\; \rho\,\widehat{\theta}_j(k) \;+\; \gamma\,r^{(k)}_j \;+\; \epsilon,
\]
where $\text{popularity}_j$ denotes a baseline popularity prior, $\widehat{\theta}_j(k)$ is the caller’s noisy estimate of competence, and $r^{(k)}_j$ is the disseminated AgentRank-UC score. The coefficients $\rho$ and $\gamma$ weight the influence of competence proxies and rank, while $\epsilon$ introduces a small exploration probability.  

This design ensures that AgentRank-UC directly shapes the flow of calls in the ecosystem: high-ranked agents are more likely to be selected, while exploration and priors preserve fairness and cold-start opportunities. Over successive epochs, telemetry and rank co-evolve, producing a feedback loop in which competence-aware discovery improves system utility while remaining robust to shocks and manipulation.

\section{Theoretical Guarantees for AgentRank-UC}

\paragraph{Overview.}
We establish five core guarantees for AgentRank-UC: (i) existence, uniqueness, and convergence of the usage and competence fixed points; (ii) well-posedness and continuity of the geometric fusion; (iii) monotonicity with respect to improved outcomes; (iv) cold-start fairness under positive priors; and (v) stability under recency decay and small perturbations. We also include an illustrative result on Sybil non-amplification.

\subsection{Preliminaries and Standing Assumptions}
Let $P,Q\in\mathbb{R}^{n\times n}$ be row-stochastic kernels constructed as in Section~3 from time-decayed sufficient statistics. Let $\alpha,\beta\in(0,1)$ and strictly positive priors $v,w\in\Delta^{n-1}$. Define the PageRank-style operators
\[
T_P(x) \coloneqq \alpha P^\top x + (1-\alpha) v,\qquad
T_Q(y) \coloneqq \beta Q^\top y + (1-\beta) w.
\]
Utilities $u_{ij}^{(k)}$ are isotone: nondecreasing in success/quality and nonincreasing in latency/cost/risk; $\phi$ (softplus) is increasing. The final score is the geometric fusion
\[
r \;=\; \mathrm{normalize}\!\big(x^{\,p}\odot y^{\,1-p}\big),\qquad p\in[0,1].
\]

\subsection{Existence, Uniqueness, and Convergence}
\begin{theorem}[Existence, uniqueness, and convergence of AgentRank-UC fixed points]\label{thm:agent-fp}
In AgentRank-UC, the usage vector $x$ and competence vector $y$ are defined by the PageRank-style equations
\[
x \;=\; \alpha P^\top x + (1-\alpha)v, \qquad
y \;=\; \beta Q^\top y + (1-\beta)w,
\]
where $P,Q$ are row-stochastic kernels constructed from time-decayed telemetry, $\alpha,\beta \in (0,1)$ are teleport weights, and $v,w\in \Delta^{n-1}$ are strictly positive priors. Then:
\begin{enumerate}\itemsep0.2em
\item (\emph{Existence and uniqueness}) Each equation admits a unique solution $x^\star,y^\star \in \Delta^{n-1}$.
\item (\emph{Convergence}) Starting from any initialization $x^{(0)},y^{(0)}\in\Delta^{n-1}$, the power iterations
\[
x^{(t+1)} = \alpha P^\top x^{(t)} + (1-\alpha)v,\qquad
y^{(t+1)} = \beta Q^\top y^{(t)} + (1-\beta)w
\]
converge linearly to $x^\star$ and $y^\star$.
\end{enumerate}
\end{theorem}

\begin{proof}
\emph{Proof overview.} We show that each update map in AgentRank-UC is a contraction on the simplex with modulus strictly less than one. By the Banach fixed-point theorem, a unique fixed point exists and the power iteration converges at a linear rate. We carry out the argument first for the usage vector $x$; the competence vector $y$ follows identically. See standard treatments of PageRank-style affine contractions and Markov chain analysis for background \cite{LangvilleMeyer2006,Levin2017,HornJohnson2013}.

\smallskip
\emph{Step 1 (Well-posedness).} For any $x\in\Delta^{n-1}$, $P^\top x$ is nonnegative and sums to one because $P$ is row-stochastic. Thus $T(x) = \alpha P^\top x + (1-\alpha)v$ is a convex combination of two simplex vectors, hence also in $\Delta^{n-1}$. Strict positivity of $v$ implies $T(x)>0$.

\smallskip
\emph{Step 2 (Contraction property).} For any $x,x'\in\mathbb{R}^n$,
\[
T(x)-T(x') = \alpha P^\top(x-x').
\]
By nonexpansiveness of column-stochastic maps in $\ell_1$, $\|P^\top(x-x')\|_1 \le \|x-x'\|_1$. Therefore
\[
\|T(x)-T(x')\|_1 \;\le\; \alpha \|x-x'\|_1,
\]
with strict contraction constant $\alpha<1$.

\smallskip
\emph{Step 3 (Existence, uniqueness, and convergence).} Since $(\Delta^{n-1},\|\cdot\|_1)$ is a complete metric space and $T$ is a strict contraction, Banach’s fixed-point theorem guarantees a unique solution $x^\star\in\Delta^{n-1}$ with $T(x^\star)=x^\star$. Moreover, starting from any $x^{(0)}$, the iterates satisfy
\[
\|x^{(t)}-x^\star\|_1 \;\le\; \alpha^t \|x^{(0)}-x^\star\|_1,
\]
establishing linear convergence. The argument for $y$ is identical with parameters $(Q,\beta,w)$.
\end{proof}

\begin{remark}[Closed form and Neumann series]
The fixed point can be written explicitly as
\[
x^\star = (1-\alpha)(I-\alpha P^\top)^{-1}v
\;=\; (1-\alpha)\sum_{t=0}^\infty \alpha^t (P^\top)^t v,
\]
with analogous expression for $y^\star$. The Neumann series converges absolutely since $\rho(\alpha P^\top)\le \alpha<1$.
\end{remark}

\begin{remark}[Iteration complexity]
From the error bound $\|x^{(t)}-x^\star\|_1 \le \alpha^t\|x^{(0)}-x^\star\|_1$, any target accuracy $\varepsilon>0$ is achieved after
\[
t \;\ge\; \frac{\ln(\varepsilon/\|x^{(0)}-x^\star\|_1)}{\ln \alpha}
\]
iterations (with $\ln\alpha<0$). Thus the cost is $O(\log 1/\varepsilon)$ iterations, each dominated by a sparse matrix–vector multiply.
\end{remark}

\subsection{Well-Posed Fusion and Continuity}
\begin{theorem}[Well-posedness and continuity of the AgentRank-UC fusion]\label{thm:fusion}
Given usage and competence fixed points $x^\star,y^\star \in \Delta^{n-1}$ with strictly positive coordinates, and a balance parameter $p\in[0,1]$, the fused AgentRank-UC score
\[
r \;=\; \mathrm{normalize}\!\left( (x^\star)^{p} \odot (y^\star)^{\,1-p} \right)
\]
is a well-defined probability vector in $\Delta^{n-1}$. Moreover, $r$ depends continuously on $(x^\star,y^\star)$ and is locally Lipschitz on any compact subset of the strictly positive orthant.
\end{theorem}

\begin{proof}
\emph{Proof overview.} We show that the geometric fusion is always positive, can be normalized to sum to one, and varies smoothly with its inputs. The argument uses only basic properties of the strictly positive orthant. Continuity and Lipschitz arguments follow standard recipes in analysis on positive orthants \cite{HornJohnson2013}.

\smallskip
\emph{Step 1 (Positivity and normalization).} Since $x^\star,y^\star>0$ and $p\in[0,1]$, each coordinate of the unnormalized vector $z=(x^\star)^p \odot (y^\star)^{\,1-p}$ is strictly positive. Normalization
\[
r_i = \frac{z_i}{\sum_j z_j}
\]
yields $r\in\Delta^{n-1}$. Thus the fusion is always well-defined.

\smallskip
\emph{Step 2 (Continuity).} The map $(x,y)\mapsto x^p \odot y^{\,1-p}$ is continuous on the positive orthant because coordinatewise exponentiation and multiplication are continuous. Division by the strictly positive sum $\sum_j z_j$ is also continuous. Hence $r$ is continuous in $(x,y)$.

\smallskip
\emph{Step 3 (Local Lipschitzness).} On any compact subset of the strictly positive orthant, all coordinates of $x,y$ are bounded away from zero and infinity. The Jacobian of $(x,y)\mapsto r$ is therefore bounded on this set, yielding a local Lipschitz constant. This ensures perturbations in $x,y$ translate to proportionally bounded changes in $r$.
\end{proof}

\begin{remark}[Perturbation stability]
Combining Theorem~\ref{thm:agent-fp} and Theorem~\ref{thm:fusion}, if $P,Q$ are perturbed by $\Delta P,\Delta Q$ of $\ell_1$ norm at most $\varepsilon_P,\varepsilon_Q$, the resulting fused vector $r$ changes by $O(\varepsilon_P+\varepsilon_Q)$. Thus AgentRank-UC is stable under small changes in input telemetry.
\end{remark}

\begin{remark}[Geometric mean vs.\ alternatives]
The multiplicative fusion penalizes imbalance: an agent that is highly popular but incompetent (large $x^\star$ but small $y^\star$) or the reverse will have a suppressed fused score. This contrasts with arithmetic averaging, which can over-reward one-sided strength. The continuity and Lipschitz properties established above ensure this penalty is applied smoothly.
\end{remark}

\subsection{Monotonicity with Respect to Improved Outcomes}
\begin{theorem}[Monotonicity of AgentRank-UC with respect to improved outcomes]\label{thm:monotonicity}
In AgentRank-UC, let the sufficient statistics on edge $(i,j,k)$ be aggregated from time-decayed interactions: successes $S_{ij}^{(k)}$, qualities $\bar q_{ij}^{(k)}$, latencies $\bar \ell_{ij}^{(k)}$, costs $\bar c_{ij}^{(k)}$, and risks $\bar r_{ij}^{(k)}$. Let competence weights be
\[
C_{ij} \;=\; \sum_{k} N_{ij}^{(k)}\, \phi\!\big(u_{ij}^{(k)}\big),
\qquad
u_{ij}^{(k)} = \theta_1 \,\mathrm{logit}\,\widehat{p}_{ij}^{(k)} - \theta_2 \log(1+\bar\ell_{ij}^{(k)}) - \theta_3 \log(1+\bar c_{ij}^{(k)}) - \theta_4 \bar r_{ij}^{(k)} + \theta_5 \bar q_{ij}^{(k)},
\]
where $\phi=\log(1+e^x)$ is the softplus function and $\widehat{p}_{ij}^{(k)}$ is the smoothed success estimate. Then increasing the success rate or quality on $(i,j,k)$, or decreasing its latency, cost, or risk, weakly increases the final fused rank $r_j$ of agent $j$.
\end{theorem}

\begin{proof}
\emph{Proof overview.} We show monotonicity in three stages: (i) edge-level utility $u_{ij}^{(k)}$ is monotone in each outcome variable, (ii) competence weights $C_{ij}$ and the fixed point $y^\star$ are monotone in $u$, and (iii) the fusion $r$ is monotone in $y$.

\smallskip
\emph{Step 1 (Utility monotonicity).} The logit of the smoothed success rate $\widehat{p}_{ij}^{(k)}$ is nondecreasing in $S_{ij}^{(k)}$; $\bar q_{ij}^{(k)}$ enters with positive coefficient; $\bar \ell_{ij}^{(k)}, \bar c_{ij}^{(k)}, \bar r_{ij}^{(k)}$ enter with negative coefficients. Thus $u_{ij}^{(k)}$ is monotone increasing in success and quality, and monotone decreasing in latency, cost, and risk.

\smallskip
\emph{Step 2 (Competence weights and fixed point).} The softplus $\phi$ is strictly increasing, so $C_{ij}$ is monotone in $u_{ij}^{(k)}$. Row-normalization preserves monotonicity in the target column $j$: increasing $C_{ij}$ raises the relative share of mass directed to $j$. The operator $T_Q(y) = \beta Q^\top y + (1-\beta)w$ is monotone on the positive cone, hence its fixed point $y^\star$ is monotone in each entry of $Q$. Therefore, improving edge $(i,j,k)$ increases $y^\star_j$ weakly.

\smallskip
\emph{Step 3 (Fusion monotonicity).} The fusion is $r = \mathrm{normalize}( (x^\star)^p \odot (y^\star)^{1-p} )$. For fixed $x^\star$, this map is coordinatewise increasing in $y^\star$, so $r_j$ weakly increases when $y^\star_j$ increases.

Combining steps 1–3, improvements in outcomes on edge $(i,j,k)$ weakly increase the final rank $r_j$.
\end{proof}

\begin{remark}[No penalty for improvement]
AgentRank-UC satisfies the desirable property that no agent can be harmed in ranking by improving its success rate or reducing latency, cost, or risk. This eliminates perverse incentives present in some heuristic scoring schemes.
\end{remark}

\begin{remark}[Strict vs.\ weak monotonicity]
If at least one outcome improves strictly (e.g., success probability rises) while others are unchanged, then $r_j$ increases strictly. If outcomes improve only marginally or in ways exactly canceled by normalization, $r_j$ may remain unchanged but never decreases.
\end{remark}

\subsection{Cold-Start Fairness}
\begin{theorem}[Cold-start fairness in AgentRank-UC]\label{thm:coldstart}
In AgentRank-UC, the usage and competence fixed points are defined by
\[
x^\star \;=\; \alpha P^\top x^\star + (1-\alpha)v, 
\qquad 
y^\star \;=\; \beta Q^\top y^\star + (1-\beta)w,
\]
where $\alpha,\beta\in(0,1)$, $P,Q$ are row-stochastic kernels, and $v,w\in \Delta^{n-1}$ are strictly positive priors. Then every coordinate of $x^\star,y^\star$ is strictly positive. Consequently, the fused AgentRank-UC score
\[
r \;=\; \mathrm{normalize}\!\big((x^\star)^{p}\odot (y^\star)^{\,1-p}\big),\qquad p\in[0,1],
\]
also satisfies $r_j>0$ for every agent $j$. In particular, newcomers with no historical calls retain strictly positive visibility through the priors.
\end{theorem}

\begin{proof}
\emph{Proof overview.} We show first that the teleportation term forces strict positivity in $x^\star$ and $y^\star$, even if some rows of $P$ or $Q$ are empty. We then argue that geometric fusion preserves positivity.

\smallskip
\emph{Step 1 (Strict positivity of $x^\star$).} By definition,
\[
x^\star \;=\; (1-\alpha)v \;+\; \alpha P^\top x^\star.
\]
Since $v>0$, the vector $(1-\alpha)v$ has strictly positive coordinates. The second term $\alpha P^\top x^\star$ is nonnegative. Hence each coordinate of $x^\star$ is strictly positive.

\smallskip
\emph{Step 2 (Strict positivity of $y^\star$).} The same argument applies with parameters $(Q,\beta,w)$. Thus $y^\star>0$ componentwise.

\smallskip
\emph{Step 3 (Fusion positivity).} The product $(x^\star)^p\odot (y^\star)^{1-p}$ is strictly positive componentwise, and normalization divides by a positive constant. Therefore $r>0$.

\smallskip
\emph{Conclusion.} All agents, including those with no observed interactions, receive strictly positive rank mass $r_j$ through the influence of positive priors $v,w$.
\end{proof}

\begin{remark}[Minimum visibility guarantee]
Every coordinate is bounded below by the prior contribution. For example,
\[
x^\star_j \;\ge\; (1-\alpha)v_j,\qquad
y^\star_j \;\ge\; (1-\beta)w_j.
\]
Thus no agent can vanish entirely from the rankings; newcomers always retain a guaranteed share of visibility.
\end{remark}

\begin{remark}[Design implication]
By choosing priors $v,w$ to reflect external certifications or attested identities, the system can boost trusted newcomers at entry. Alternatively, uniform priors ensure purely egalitarian cold-start treatment.
\end{remark}

\subsection{Stability Under Recency and Small Perturbations}
\begin{theorem}[Stability of AgentRank-UC under recency decay and small perturbations]\label{thm:stability}
Let the usage and competence kernels $P,Q$ be built from time–decayed telemetry as in Section~3, with exponentially decayed sufficient statistics and row–normalization (with prior backoff on empty rows). Fix $\alpha,\beta\in(0,1)$ and strictly positive priors $v,w\in\Delta^{n-1}$. Then:
\begin{enumerate}\itemsep0.2em
\item \textbf{Preservation under decay.} For any decay rate $\lambda>0$ (half–life $H=\ln 2/\lambda$), the kernels $P,Q$ are row–stochastic. Hence the AgentRank–UC fixed points $x^\star,y^\star$ are well–defined and unique by Theorem~\ref{thm:agent-fp}, and the power iterations converge linearly with rates at most $\alpha$ and $\beta$.
\item \textbf{Perturbation stability.} Let $\widetilde P,\widetilde Q$ be kernels obtained from (possibly) different decayed statistics or a different decay rate $\widetilde\lambda$, and let $\widetilde x^\star,\widetilde y^\star$ denote the corresponding fixed points. Then
\[
\|x^\star-\widetilde x^\star\|_1 \;\le\; \frac{\alpha}{1-\alpha}\,\|P-\widetilde P\|_1,
\qquad
\|y^\star-\widetilde y^\star\|_1 \;\le\; \frac{\beta}{1-\beta}\,\|Q-\widetilde Q\|_1.
\]
Consequently, the fused score $r=\mathrm{normalize}\!\big((x^\star)^p\odot (y^\star)^{1-p}\big)$ depends continuously on $(P,Q)$ and is Lipschitz on compact subsets of the strictly positive orthant.
\end{enumerate}
\end{theorem}

\begin{proof}
\emph{Proof overview.} We first note that exponential decay rescales nonnegative edge weights prior to row–normalization; this preserves stochasticity and therefore the well–posedness and convergence guarantees from Theorem~\ref{thm:agent-fp}. We then derive an explicit perturbation bound for the fixed points by comparing the two contractive maps associated with $(P,\alpha,v)$ and $(\widetilde P,\alpha,v)$ (and analogously for $Q$). Row-normalization Lipschitz behavior and perturbation consequences for stationary distributions are classical \cite{Seneta2006,Levin2017}.

\smallskip
\emph{Step 1 (Decay preserves stochasticity and convergence).}
Let $U$ (usage weights) and $C$ (competence weights) be the nonnegative matrices produced from the decayed sufficient statistics; row–normalization with prior backoff yields row–stochastic $P=\mathrm{rownorm}(U)$ and $Q=\mathrm{rownorm}(C)$. Exponential decay $\omega(t)=e^{-\lambda(T-t)}$ only rescales nonnegative contributions and does not introduce negative entries; empty rows are replaced by priors, so each row sums to one. Therefore the PageRank–style operators
\[
T_P(x)=\alpha P^\top x+(1-\alpha)v,\qquad T_Q(y)=\beta Q^\top y+(1-\beta)w
\]
are strict contractions in $\ell_1$ with moduli $\alpha$ and $\beta$ (unchanged by $\lambda$). Theorem~\ref{thm:agent-fp} gives existence, uniqueness, and linear convergence of $x^\star,y^\star$.

\smallskip
\emph{Step 2 (Perturbation bound for $x^\star$).}
Let $\widetilde T(x)=\alpha \widetilde P^\top x+(1-\alpha)v$ and let $x^\star,\widetilde x^\star$ denote the respective fixed points. Subtract the fixed–point equations:
\[
x^\star-\widetilde x^\star
= \alpha P^\top x^\star - \alpha \widetilde P^\top \widetilde x^\star
= \alpha P^\top (x^\star-\widetilde x^\star) \;+\; \alpha (P^\top-\widetilde P^\top)\,\widetilde x^\star.
\]
Taking $\ell_1$ norms and using $\|P^\top w\|_1\le \|w\|_1$ and $\|\widetilde x^\star\|_1=1$ yields
\[
\|x^\star-\widetilde x^\star\|_1
\;\le\; \alpha \|x^\star-\widetilde x^\star\|_1 \;+\; \alpha \|P-\widetilde P\|_1.
\]
Rearranging gives the claimed bound
\[
\|x^\star-\widetilde x^\star\|_1 \;\le\; \frac{\alpha}{1-\alpha}\,\|P-\widetilde P\|_1.
\]
The same argument holds for $y^\star$ with $(Q,\beta)$ in place of $(P,\alpha)$.

\smallskip
\emph{Step 3 (Continuity and Lipschitz fusion).}
By Theorem~\ref{thm:fusion}, the map $(x,y)\mapsto r$ is continuous and locally Lipschitz on the strictly positive orthant. Combining with Step~2 shows that small changes in $P,Q$ (for example, induced by small changes in the decay rate $\lambda$ or in the decayed statistics) lead to proportionally bounded changes in $x^\star,y^\star$, and hence in $r$.
\end{proof}

\begin{remark}[Continuity in the half–life]
If $\lambda$ is replaced by $\widetilde\lambda$, then for each edge weight the decayed sum changes by at most a factor proportional to $\sup_\tau |e^{-\lambda \tau}-e^{-\widetilde\lambda \tau}|$. Row–normalization is continuous on the positive cone (with prior backoff ensuring strictly positive row sums), so $\|P(\lambda)-P(\widetilde\lambda)\|_1$ and $\|Q(\lambda)-Q(\widetilde\lambda)\|_1$ can be made arbitrarily small as $|\lambda-\widetilde\lambda|\to 0$. The bound in Theorem~\ref{thm:stability} then yields continuity of $x^\star,y^\star,r$ in $\lambda$.
\end{remark}

\begin{remark}[Convergence rate is independent of recency]
The linear convergence rates of the power iterations are controlled by $\alpha$ and $\beta$ (the contraction moduli), not by $\lambda$. Thus changing the half–life trades off \emph{responsiveness} versus \emph{variance} in the input statistics without affecting the existence/uniqueness or the asymptotic rate of the fixed–point solver.
\end{remark}

\begin{remark}[From input statistics to kernels]
Let $U,\widetilde U$ be two nonnegative usage–weight matrices with row sums bounded away from zero (or backed off to priors). Then the row–normalization map satisfies a Lipschitz–type bound $\| \mathrm{rownorm}(U)-\mathrm{rownorm}(\widetilde U)\|_1 \le K\,\|U-\widetilde U\|_1$ for some $K$ depending on the minimal row sum. An analogous bound holds for competence weights $C$. Together with Step~2, this yields end–to–end stability from telemetry perturbations to the fused rank $r$.
\end{remark}

\subsection{Sybil Non-Amplification}
\begin{theorem}[Sybil non-amplification in AgentRank-UC]\label{thm:sybil}
Let $S\subseteq\{1,\dots,n\}$ be a collusive subset (“Sybil clique”). Write
\(
x_S \coloneqq \sum_{j\in S} x^\star_j,\;
y_S \coloneqq \sum_{j\in S} y^\star_j,\;
r_S \coloneqq \sum_{j\in S} r_j,\;
v_S \coloneqq \sum_{j\in S} v_j
\)
and let $v_{\min,S^c} \coloneqq \min_{j\notin S} v_j$.
Assume $P$ and $Q$ are formed from time-decayed telemetry as in Section~3, with strictly positive priors $v,w$ and $\alpha,\beta\in(0,1)$, and the fused score
\(
r=\mathrm{normalize}\!\big((x^\star)^p\odot (y^\star)^{1-p}\big)
\)
for some $p\in(0,1]$.
Then:

\begin{enumerate}\itemsep0.2em
\item[\textbf{(A)}] \textbf{Usage-only amplification bound.}
For any row-stochastic $P$ (including extreme intra-clique usage pumping),
\begin{equation}\label{eq:xs-sandwich}
(1-\alpha)\,v_S \;\le\; x_S \;\le\; \alpha \;+\; (1-\alpha)\,v_S.
\end{equation}
In particular, even if the clique tries to trap usage mass, $x_S$ cannot exceed $\alpha+(1-\alpha)v_S$.

\item[\textbf{(B)}] \textbf{UC bound with unchanged competence.}
If collusion does not improve competence (i.e., $Q$ is unchanged on $S$ so that $y_S<1$), then for any $p\in(0,1]$,
\begin{equation}\label{eq:RS-bound}
r_S \;\le\;
\frac{\big(\alpha+(1-\alpha)v_S\big)^{p}\; y_S^{\,1-p}}
     {\big(\alpha+(1-\alpha)v_S\big)^{p}\; y_S^{\,1-p}
      \;+\; (1-\alpha)^{p}\; (v_{\min,S^c})^{p}\; (1-y_S)^{\,1-p}}.
\end{equation}
Consequently, if $v_{\min,S^c}>0$ and $y_S<1$, then $r_S<1$, so increasing intra-clique usage alone cannot drive the fused rank mass to $1$.
\end{enumerate}
\end{theorem}

\begin{proof}
\emph{Proof overview.} Part (A) is a direct mass-balance argument for the usage fixed-point $x^\star$: teleportation injects $(1-\alpha)v$ each step and transitions can send at most total mass $1$ into $S$; this yields a sharp sandwich bound for $x_S$. Part (B) upper-bounds the UC numerator over $S$ using Hölder-type aggregation and lower-bounds the outside contribution using the positive teleportation prior on $S^c$, which ensures a nonzero floor on $x$ outside $S$. The bound \eqref{eq:RS-bound} then follows by grouping $S$ versus $S^c$. Graph-based defenses against Sybil amplification in related contexts motivate our non-amplification perspective \cite{Yu2006,Yu2008,Kamvar2003}.

\smallskip
\noindent\emph{(A) Usage-only amplification bound.}
Sum the usage fixed-point equation $x^\star=(1-\alpha)v+\alpha P^\top x^\star$ over $j\in S$:
\[
x_S \;=\; (1-\alpha) v_S \;+\; \alpha\,\sum_{i=1}^n x^\star_i\, P_{iS},
\qquad P_{iS}\coloneqq \sum_{j\in S} P_{ij}.
\]
Since $0\le P_{iS}\le 1$ and $\sum_i x^\star_i=1$, we have
\[
(1-\alpha) v_S \;\le\; x_S \;\le\; (1-\alpha) v_S + \alpha\cdot 1
\;=\; \alpha+(1-\alpha) v_S,
\]
which proves \eqref{eq:xs-sandwich}. Both bounds are tight (e.g., no inbound/outbound into $S$ gives $x_S=(1-\alpha)v_S$; everyone transitions into $S$ gives $x_S=\alpha+(1-\alpha)v_S$).

\smallskip
\noindent\emph{(B) UC bound with unchanged competence.}
Write the fused mass over $S$ as
\[
r_S \;=\; \frac{\sum_{j\in S} (x^\star_j)^{p}(y^\star_j)^{1-p}}
{\sum_{j\in S} (x^\star_j)^{p}(y^\star_j)^{1-p}
 \;+\; \sum_{j\notin S} (x^\star_j)^{p}(y^\star_j)^{1-p}}.
\]
\emph{Upper bound the numerator:} concavity of $t\mapsto t^{p}$ on $[0,1]$ and Hölder’s inequality yield the aggregate bound
\[
\sum_{j\in S} (x^\star_j)^{p}(y^\star_j)^{1-p}
\;\le\; \Big(\sum_{j\in S} x^\star_j\Big)^{p}\;
       \Big(\sum_{j\in S} y^\star_j\Big)^{1-p}
\;=\; x_S^{p}\, y_S^{\,1-p}
\;\le\; \big(\alpha+(1-\alpha)v_S\big)^{p}\, y_S^{\,1-p},
\]
where the last inequality uses the usage bound $x_S\le \alpha+(1-\alpha)v_S$ from (A).

\emph{Lower bound the denominator’s outside term:} for each $j\notin S$, strict positivity of the prior implies $x^\star_j\ge (1-\alpha)v_j$, so
\[
\sum_{j\notin S} (x^\star_j)^{p}(y^\star_j)^{1-p}
\;\ge\; \sum_{j\notin S} \big((1-\alpha)v_j\big)^{p} (y^\star_j)^{1-p}
\;\ge\; (1-\alpha)^{p} (v_{\min,S^c})^{p} \sum_{j\notin S} (y^\star_j)^{1-p}
\;\ge\; (1-\alpha)^{p} (v_{\min,S^c})^{p}\, (1-y_S)^{\,1-p}.
\]
Combining the two bounds and simplifying gives \eqref{eq:RS-bound}. Since $v_{\min,S^c}>0$ and $y_S<1$, the denominator strictly exceeds the numerator, hence $r_S<1$.
\end{proof}

\begin{remark}[What the bound says operationally]
The usage-only share $x_S$ can be made large by collusion, but teleportation enforces the hard ceiling $x_S\le \alpha+(1-\alpha)v_S$. AgentRank-UC mixes $x^\star$ with competence $y^\star$: if the clique’s competence share $y_S$ is bounded away from $1$ (no real improvement), the fused mass $r_S$ remains strictly $<1$, with an explicit gap controlled by the outside prior floor $v_{\min,S^c}$ and by $1-y_S$.
\end{remark}

\begin{remark}[Uniform priors and a concrete gap]
If $v$ is uniform ($v_j=1/n$) and $p\in(0,1)$, then
\[
r_S \;\le\; 
\frac{\big(\alpha+(1-\alpha)\tfrac{|S|}{n}\big)^{p}\; y_S^{\,1-p}}
{\big(\alpha+(1-\alpha)\tfrac{|S|}{n}\big)^{p}\; y_S^{\,1-p}
 +(1-\alpha)^{p}\big(\tfrac{1}{n}\big)^{p}\, (1-y_S)^{\,1-p}},
\]
so even aggressive usage pumping (large $\alpha$, dense intra-clique calls) cannot overwhelm a modest outside competence share when $p<1$.
\end{remark}

\begin{remark}[When the bound tightens]
The outside floor $(1-\alpha)^{p} (v_{\min,S^c})^{p} (1-y_S)^{1-p}$ increases if either the teleportation rate $(1-\alpha)$, the minimum outside prior $v_{\min,S^c}$, or the outside competence share $(1-y_S)$ increases, all of which further limit Sybil mass. Identity-weighted priors (larger $v$ for verified agents) strengthen this effect.
\end{remark}

\section{Simulation of AgentRank}

\subsection{Goals of Simulation}

The purpose of our simulations is to evaluate whether AgentRank-UC achieves the properties claimed in Section~3 under controlled, synthetic conditions. Since no large-scale ecosystem of interacting agents with standardized telemetry yet exists, simulation provides the only principled way to test both the algorithm’s behavior and its robustness to adversarial scenarios. Our goals are fourfold.  

First, we aim to demonstrate the \emph{dual-signal advantage}: by combining usage and competence signals, AgentRank-UC should surface higher-quality agents than methods that rely on either signal in isolation. This validates the central premise of the algorithm---that popularity and demonstrated performance must be considered together.  

Second, we seek to show that the balance parameter $p$ provides an \emph{interpretable trade-off} between usage and competence. By sweeping $p$ across its range, we expect to observe smooth transitions between rankings that reflect popularity alone and those that reflect competence alone, with intermediate values capturing an effective compromise.  

Third, we test the algorithm’s \emph{responsiveness and fairness}. With exponential decay, AgentRank-UC should adapt to shocks such as sudden improvements or degradations in agent competence. At the same time, the algorithm should satisfy monotonicity, meaning that improved success rates always increase rank, and cold-start agents should receive non-negligible visibility through priors.  

Finally, we evaluate \emph{robustness to manipulation}. In particular, we simulate Sybil cliques and usage-pumping attacks in which colluding agents attempt to inflate usage without genuine competence. AgentRank-UC should resist such manipulation more effectively than pure usage-based rankings.  

Together, these goals provide a comprehensive test of whether AgentRank-UC delivers on its promises of accuracy, adaptability, fairness, and robustness, thereby establishing confidence in its viability for real-world deployment.  We follow standard ranking-evaluation conventions (e.g., NDCG, top-$k$ quality) as used in graph-ranking literature \cite{Gleich2015}. Agent-based evaluation ideas are informed by recent benchmarks and multi-agent frameworks \cite{Liu2024,Wu2023}.

\subsection{World Model and Data Generation}

To evaluate AgentRank-UC in the absence of real-world telemetry, we construct a synthetic environment that emits interaction logs consistent with the OAT-Lite schema. The model is deliberately simple: it captures competence, popularity, and cost--risk trade-offs while remaining lightweight enough to simulate on consumer-grade hardware.  

\paragraph{Agents and tasks.}  
We instantiate $n$ agents and $d$ task types. Each agent $j$ is assigned latent parameters for each task $k$: a competence level $\theta_j(k) \in [0,1]$, a mean latency $\ell_j(k)$, a mean cost $c_j(k)$, and a risk parameter $r_j(k)$. These parameters represent the ground-truth capabilities of the agent and determine the distribution of outcomes whenever the agent is invoked.  

\paragraph{Archetypes.}  
For interpretability, agents are drawn from a small set of archetypes. Popular-but-Mediocre (PbM) agents attract high baseline traffic but achieve only moderate competence ($\theta \approx 0.5$). Niche-but-Excellent (NbE) agents are highly competent ($\theta \approx 0.9$) on one or two task types but receive little baseline traffic. Balanced-Strong (BS) agents are generally strong performers ($\theta \approx 0.8$) across all tasks. Cheap-but-Risky (CbR) agents provide low latency and cost but have elevated risk. A Sybil Clique (SY) consists of colluding agents that call each other frequently while maintaining only mediocre competence. Finally, Newcomer-Good (NC) agents are introduced midway through the simulation with high competence but no usage history. These archetypes allow us to stress-test the ranking algorithm under realistic scenarios of popularity bias, specialization, risk, collusion, and cold-start entry.

\begin{table}[h]
\centering
\caption{Agent archetypes in the simulation world. Each archetype has characteristic profiles of competence, latency, cost, and risk. Qualitative labels (Low/Medium/High) are accompanied by indicative numeric values.}
\begin{tabularx}{\linewidth}{lXXXX}
\toprule
\textbf{Archetype} 
& \textbf{Competence} 
& \textbf{Latency} 
& \textbf{Cost} 
& \textbf{Risk} \\
\midrule
\textbf{Balanced-Strong (BS)} 
& High ($\sim$0.8) 
& Medium ($\sim$300 ms) 
& Medium ($\sim$1.0) 
& Low ($\sim$0.05) \\

\textbf{Popular-but-Mediocre (PbM)} 
& Medium ($\sim$0.55) 
& Medium ($\sim$300 ms) 
& Medium ($\sim$1.0) 
& Low ($\sim$0.05) \\

\textbf{Niche-but-Excellent (NbE)} 
& High in one task ($\sim$0.9), Medium elsewhere ($\sim$0.5) 
& Medium ($\sim$270 ms) 
& Medium ($\sim$0.9) 
& Low ($\sim$0.05) \\

\textbf{Cheap-but-Risky (CbR)} 
& Medium ($\sim$0.65) 
& Low ($\sim$180 ms) 
& Low ($\sim$0.6) 
& High ($\sim$0.15) \\

\textbf{Sybil Clique (SY)} 
& Medium ($\sim$0.5) 
& Medium--High ($\sim$320 ms) 
& Medium ($\sim$1.0) 
& Medium--High ($\sim$0.10) \\
\bottomrule
\end{tabularx}
\label{tab:archetypes}
\end{table}

\begin{figure}[h]
  \centering
  \includegraphics[width=0.55\linewidth]{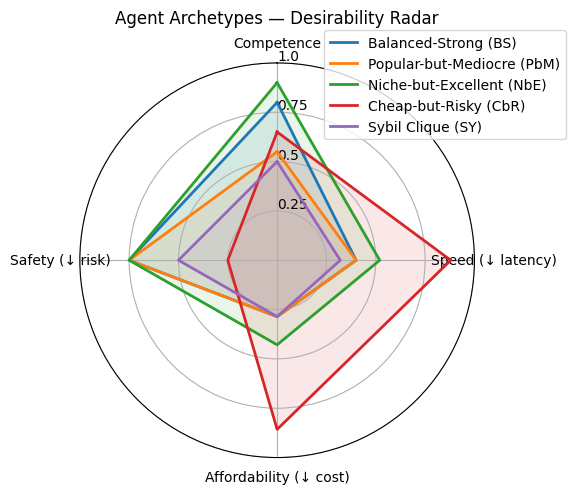}
  \caption{Desirability radar for agent archetypes across four dimensions: Competence, Speed (inverse of latency), Affordability (inverse of cost), and Safety (inverse of risk). Higher is better along each axis.}
  \label{fig:archetype-radar}
\end{figure}

\paragraph{Callers and routing.}  
All agents can act as callers. At each timestep, a caller $i$ samples a task $k$ and selects a callee $j$. We distinguish two regimes.  

\emph{Pre-rank (burn-in).} Before any ranking has been computed, callers select according to
\[
\Pr(i \to j \mid k) \propto (1-\rho)\,\text{popularity}_j + \rho\,\widehat{\theta}_j(k) + \epsilon,
\]
where $\text{popularity}_j$ is a baseline popularity prior (e.g., Zipf-distributed), $\widehat{\theta}_j(k)$ is a noisy proxy for competence, $\rho$ controls their relative weight, and $\epsilon$ ensures exploration.  

\emph{Rank-informed (after the first epoch).} Once the indexer publishes per-task ranks $r^{(k)}$, the selection rule becomes
\[
\Pr(i \to j \mid k) \propto (1-\rho-\gamma)\,\text{popularity}_j + \rho\,\widehat{\theta}_j(k) + \gamma\,r^{(k)}_j + \epsilon,
\]
where $\gamma$ weights the influence of the disseminated AgentRank-UC scores. In practice, we implement this as a softmax with temperature $\tau>0$, which preserves preference orderings while avoiding deterministic lock-in.  

\paragraph{Interaction process.}  
When a call is made, outcomes are drawn from the callee’s latent profile. A success indicator $z$ is drawn from $\mathrm{Bernoulli}(\theta_j(k))$. A quality score $q$ is sampled from a truncated Normal centered at $\theta_j(k)$. Latency $\ell$ is drawn from a log-normal distribution with mean $\ell_j(k)$, cost $c$ from a gamma distribution with mean $c_j(k)$, and risk $r$ from a Beta distribution centered on $r_j(k)$. Each interaction produces a tuple $(i \to j, k, t; z, q, \ell, c, r)$, which is aggregated into OAT-Lite records.  

\paragraph{Non-stationarity.}  
To test responsiveness, we introduce shocks during the simulation. At a designated time $t^\star$, one PbM agent experiences a competence drop, an NbE agent improves further, and an NC agent enters the ecosystem with high competence but no prior calls. These events probe how quickly rankings adapt, whether improvements translate monotonically into higher ranks, and how newcomers are incorporated.  

\medskip
This world model provides a principled but tractable testbed: it is rich enough to exercise dual-signal discovery, recency responsiveness, cold-start fairness, and Sybil resistance, yet simple enough to simulate efficiently on a laptop.

\subsection{Experimental Protocol}

We now describe the general protocol followed across all simulation experiments. Each component is standardized so that results are comparable across different parameter settings and regimes.

\begin{itemize}
    \item \textbf{Timesteps.} Simulations unfold over a fixed horizon of $T$ epochs. Within each epoch, a set number of calls is generated, and the resulting interactions are logged. Unless otherwise specified, $T=40$ epochs are used with $200$--$300$ calls per epoch. Early epochs serve as a burn-in period, after which rankings are recomputed and incorporated into caller routing.

    \item \textbf{Telemetry.} At the end of each epoch, raw interactions are aggregated into OAT-Lite telemetry reports. For each $(i,j,k)$ triple (caller $i$, callee $j$, task type $k$), the sufficient statistics include decayed counts of calls and successes, as well as decayed sums of quality, latency, cost, and risk. Exponential decay with half-life $H=\ln 2/\lambda$ ensures that recent behavior influences rankings more strongly.

    \item \textbf{Algorithm.} Given telemetry, we construct usage and competence kernels $P$ and $Q$, and compute usage and competence fixed points with teleportation parameters $(\alpha,\beta)\in(0,1)$. AgentRank-UC then combines them via the geometric merge parameter $p\in[0,1]$. Unless otherwise noted, priors $(v,w)$ are uniform. Hyperparameters $(\alpha,\beta,p,\lambda)$ are reported for each experiment.

    \item \textbf{Baselines.} To contextualize AgentRank-UC, we evaluate against usage-only (the stationary distribution of $P$), competence-only (the stationary distribution of $Q$), and the naive success rate baseline (empirical mean of successes). In some settings we also include PageRank on the usage graph as a classical reference.

    \item \textbf{Ablations.} Where informative, we test simplified variants such as removing Beta smoothing, eliminating priors (which penalizes cold-start agents), or dropping cost/latency/risk penalties. These ablations clarify which design choices are critical for robustness.

    \item \textbf{Evaluation metrics.} Discovery quality is measured by the mean true competence of the top-$k$ agents (Quality@$k$) and by normalized discounted cumulative gain (NDCG@$k$). Rank correlation with ground truth competence is assessed via Kendall-$\tau$. Decision quality is evaluated through regret (gap to an oracle that always selects the best agent for each task). Robustness is evaluated through Sybil mass (total rank assigned to collusive cliques).
\end{itemize}

\paragraph{Clean vs.\ Realistic regimes.}
Throughout our experiments we distinguish two regimes. In the \emph{clean} regime, competence is caller-independent, no noise is added to routing, there are no Sybil cliques, and penalties on latency, cost, and risk are removed. This setting provides a near-ideal stationary world in which competence-only ranking should dominate. In the \emph{realistic} regime, we introduce noise in competence proxies, random record drops, caller-dependent perturbations, and a small Sybil clique with collusive calls. We also reinstate penalties on latency, cost, and risk. This regime stresses the algorithm under conditions of sparsity, noise, and adversarial behavior, and typically reveals the benefits of a modest usage component.

\subsection{Experiments and Results}
We organize results around the core claims and goals:
\begin{itemize}
    \item \textbf{Dual-signal advantage (Exp-1).} Show AgentRank-UC versus baselines.  
    \item \textbf{Balance parameter $p$ (Exp-2).} Sweep $p$, plot trade-offs.  
    \item \textbf{Recency and shocks (Exp-3).} Show rank trajectories before and after shocks under different half-lives.  
    \item \textbf{Monotonicity and cold-start (Exp-4).} Demonstrate non-decreasing ranks and newcomer lift.  
    \item \textbf{Sybil resistance (Exp-5).} Compare mass assigned to Sybil cliques under different methods.  
    \item \textbf{Topic conditioning (Exp-6, optional).} Demonstrate specialist surfacing.  
\end{itemize}
Each experiment  briefly restates its expectation, present results (plots or tables), and provides interpretation.  

\subsection{Experiment One- Sanity Check and Comparison with Oracle Baseline}
The first experiment evaluates AgentRank-UC against usage-only, competence-only, and a naive baseline that simply reports empirical success rates for each callee (aggregated over callers). In our simulation world, competence $\theta_j(k)$ is caller-independent and outcomes are plentiful, so the naive success-rate estimator is effectively an \emph{oracle}: with enough samples it converges directly to the true underlying competence. This baseline is not realistic for open ecosystems---where telemetry is sparse, delayed, adversarial, and often caller-dependent---but it provides a useful upper bound in a controlled setting.

We instantiate a world with $n=100$ agents across $d=3$ task types, drawn from the archetypes in Section~4.2 (Popular-but-Mediocre, Niche-but-Excellent, Balanced-Strong, Cheap-but-Risky, Sybil Clique, and Newcomer-Good). The system runs for $40$ epochs with $200$ calls per epoch. After a $5$-epoch burn-in, the indexer computes per-task rankings each epoch and disseminates them; callers then incorporate these ranks in their selection rule. We evaluate the last epoch using Quality@$k$ (mean true competence of the top-$k$ agents), NDCG@$k$, Spearman’s $\rho$ with ground truth, and Regret@$k$ relative to an oracle that always selects the best agent.

\begin{figure}[!htb]
  \centering
  \begin{subfigure}[t]{0.48\textwidth}
    \centering
    \includegraphics[width=\linewidth]{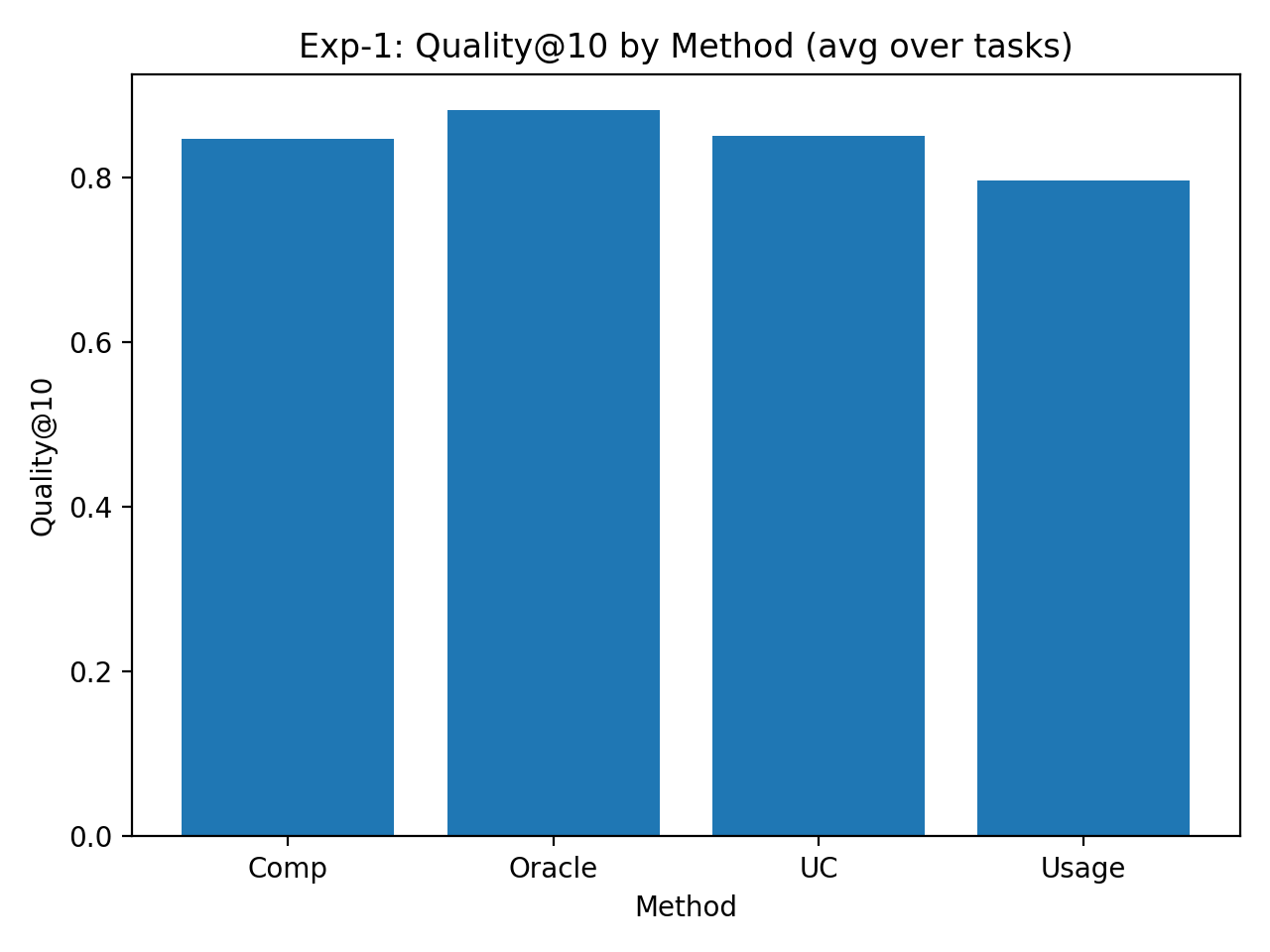}
    \caption{Quality@10 by method (avg over tasks). The naive success-rate baseline acts as an oracle in this toy setting; AgentRank-UC tracks it closely and clearly outperforms usage-only.}
    \label{fig:quality10}
  \end{subfigure}\hfill
  \begin{subfigure}[t]{0.48\textwidth}
    \centering
    \includegraphics[width=\linewidth]{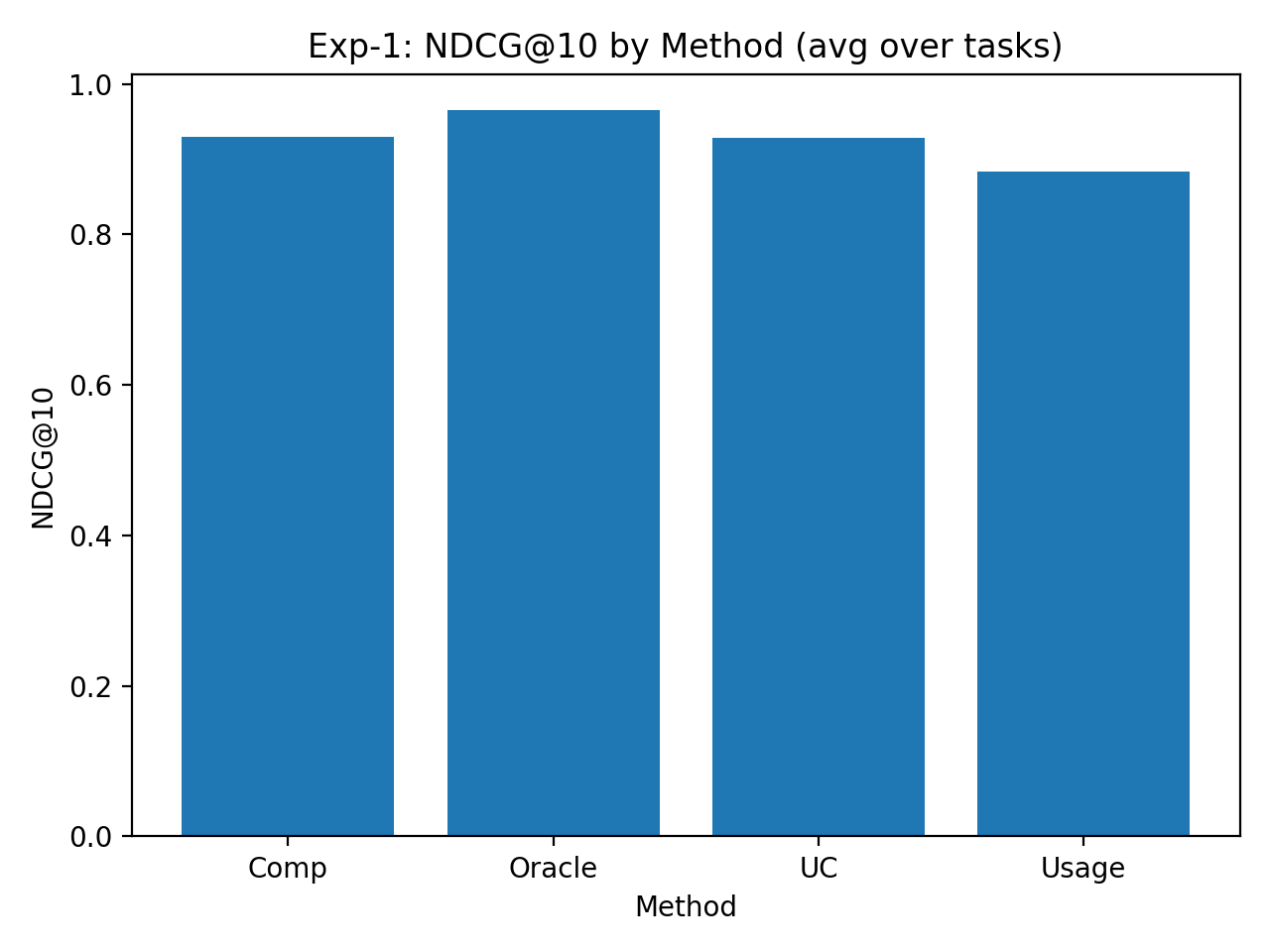}
    \caption{NDCG@10 by method (avg over tasks). AgentRank-UC is near-oracle and substantially stronger than usage-only.}
    \label{fig:ndcg10}
  \end{subfigure}
  \caption{Experiment 1 bar charts side-by-side.}
  \label{fig:exp1-bars}
\end{figure}

\paragraph{Interpretation.}
In this controlled stationary setting, the naive success-rate baseline serves as an oracle and achieves the highest top-$k$ quality. AgentRank-UC tracks this oracle closely, substantially outperforming usage-only and nearly matching competence-only. Importantly, the oracle assumes perfect observability and caller-independent competence; such assumptions do not hold in open ecosystems. Thus, while the oracle provides a useful upper bound, the results show that AgentRank-UC already operates near this ceiling while offering robustness (recency, smoothing, multi-objective utility) that naive methods lack.

\subsection{Experiment Two - Role of Balance parameter $p$ }

This experiment revisits the role of the merge parameter $p \in [0,1]$ in AgentRank-UC using a frozen-telemetry design. Our aim is to show that AgentRank-UC provides a continuous interpolation between competence-only ($p=0$) and usage-only ($p=1$). In particular, when telemetry is fixed, the competence-only and usage-only rankings form flat baselines, and the UC curve traces a smooth path between them as $p$ varies. This construction avoids feedback-loop artifacts in dynamic simulations and provides a clear visualization of how the merge parameter governs the trade-off.

We instantiate the same agent world as in Exp-1, with $n=100$ agents spanning $d=3$ task types and archetypes (Popular-but-Mediocre, Niche-but-Excellent, Balanced-Strong, Cheap-but-Risky, Sybil Clique, and Newcomer-Good). For each regime (clean vs. realistic), we first simulate interactions for 35 epochs with neutral routing (no UC in the loop), accumulating OAT-Lite telemetry. We then freeze this telemetry and build the usage kernel $P$ and competence kernel $Q$. From these, we compute the usage-only fixed point $x$ and the competence-only fixed point $y$ once. Finally, we sweep $p \in \{0.0, 0.125, \ldots, 1.0\}$ and evaluate $r(p) = \mathrm{normalize}(x^p \odot y^{1-p})$, measuring Quality@10 and NDCG@10 at each $p$.

\begin{figure}[!htb]
  \centering
  \begin{subfigure}[t]{0.48\textwidth}
    \centering
    \includegraphics[width=\linewidth]{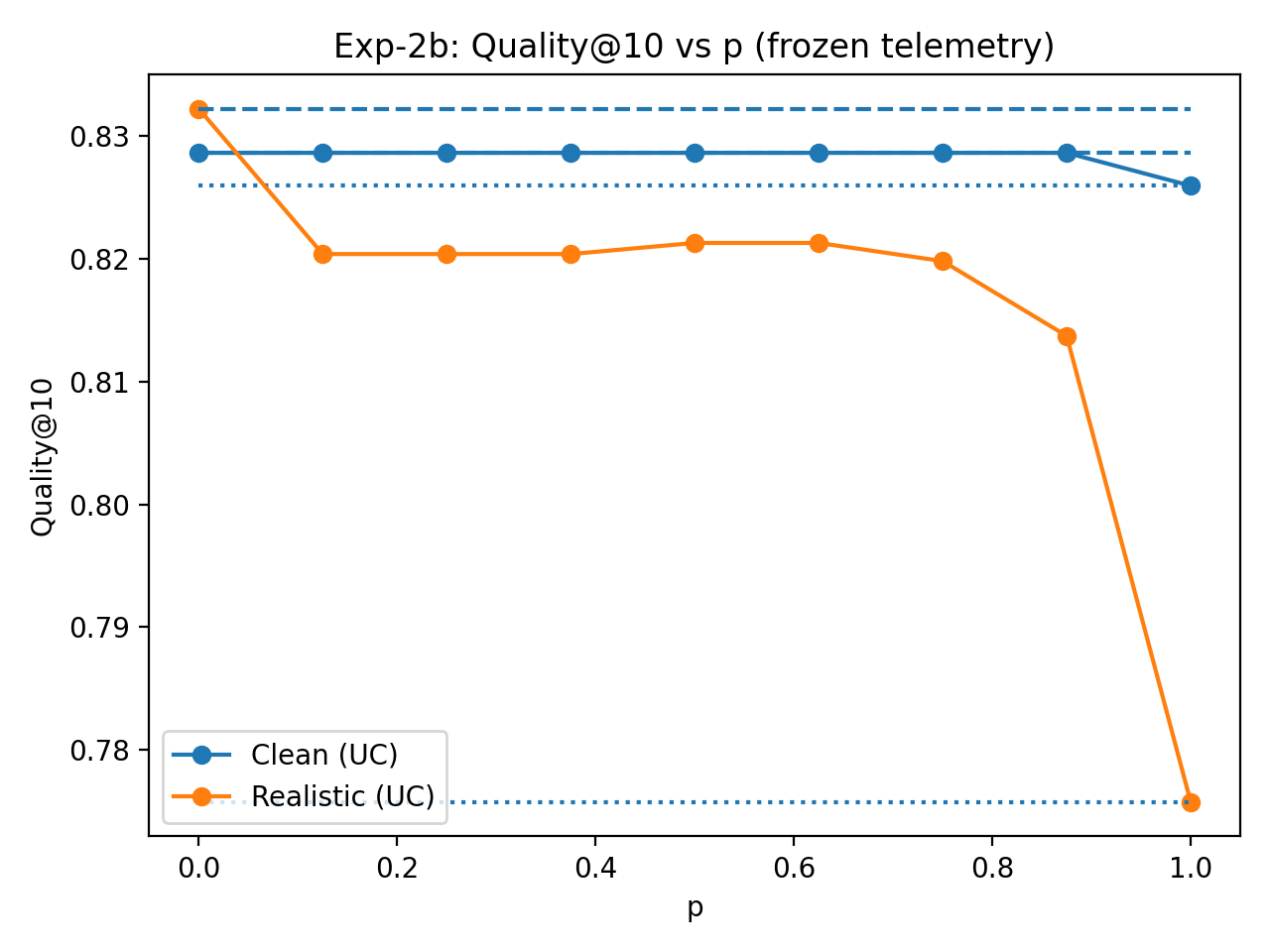}
    \caption{Quality@10 as a function of $p$, with competence-only and usage-only baselines shown as flat lines. The UC curve interpolates smoothly between them.}
    \label{fig:exp2b-quality}
  \end{subfigure}\hfill
  \begin{subfigure}[t]{0.48\textwidth}
    \centering
    \includegraphics[width=\linewidth]{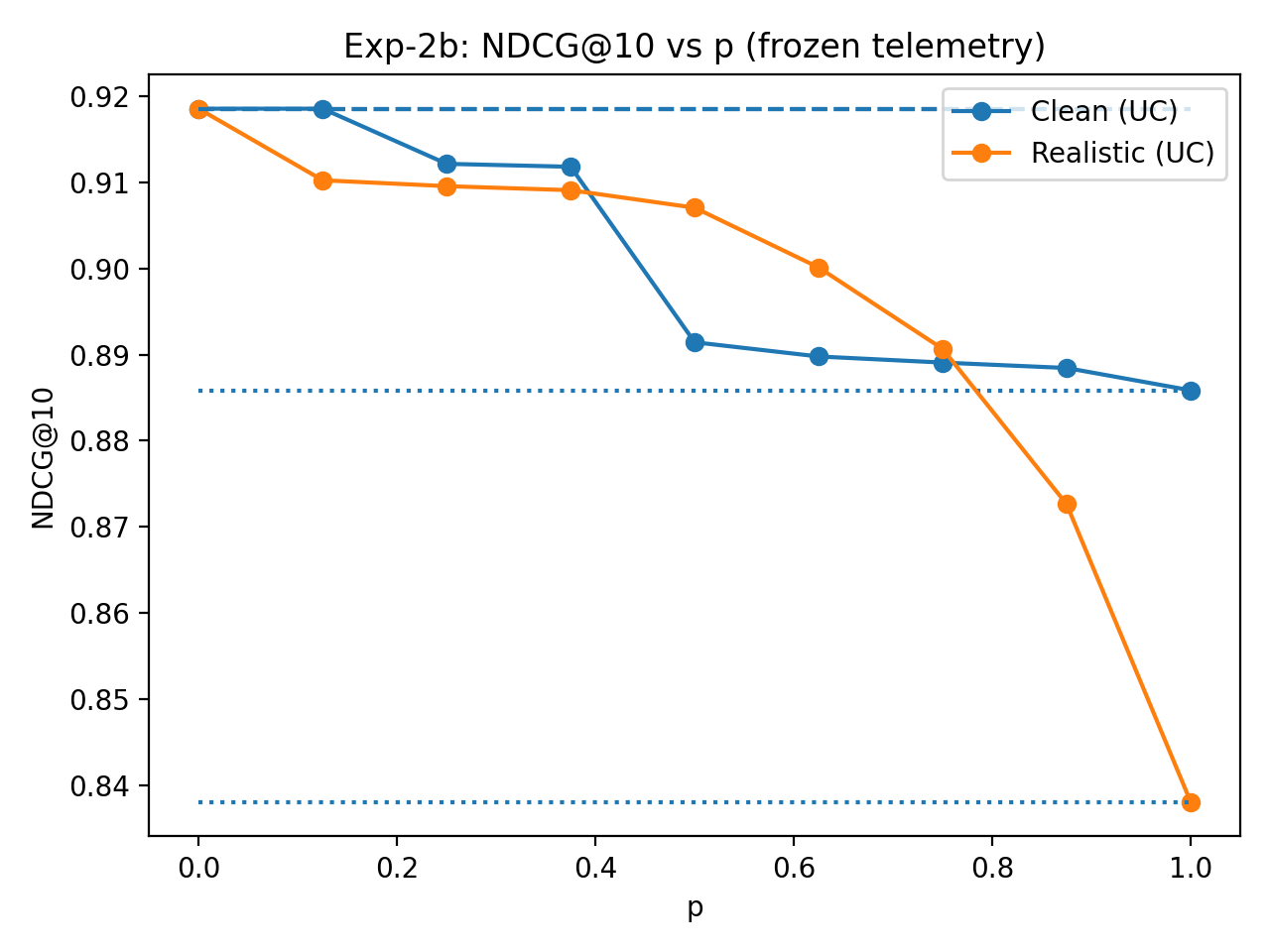}
    \caption{NDCG@10 as a function of $p$. The same interpolation pattern is observed, with the UC curve lying between the competence-only and usage-only baselines.}
    \label{fig:exp2b-ndcg}
  \end{subfigure}
  \caption{Experiment 2b results: UC interpolates between competence-only ($p=0$) and usage-only ($p=1$) baselines under both clean and realistic regimes.}
  \label{fig:exp2b}
\end{figure}

\paragraph{Interpretation.}
The frozen-telemetry setting makes clear that AgentRank-UC is a geometric interpolation between two extremes. At $p=0$, the UC curve exactly coincides with competence-only; at $p=1$, it coincides with usage-only. For intermediate $p$, the UC scores lie strictly between the baselines, rising or falling smoothly as the balance shifts. In the clean regime, competence-only dominates, while in the realistic regime modest usage mass can help stabilize exposure. This experiment thus confirms that $p$ is an interpretable and continuous control knob: UC rankings evolve predictably between the competence and usage baselines as $p$ varies.

\subsection{Experiment Three - Adaptation to Recency and shocks}

This experiment examines how AgentRank-UC adapts to sudden changes in agent performance and how the exponential decay half-life $H$ trades off responsiveness and stability. We consider two events at a designated shock epoch $t^\star$: a popular-but-mediocre (PbM) agent suffers a competence drop, and a niche-but-excellent (NbE) specialist improves further on its specialty task. Our goal is to show that shorter half-lives adapt more quickly to changes, while longer half-lives provide smoother but slower responses.

We simulate the world from Section~4.2 with $n=100$ agents and $d=3$ tasks, running for $40$ epochs with a $5$-epoch burn-in. At epoch $t^\star=18$, the PbM agent’s competence is reduced by $-0.2$ across tasks, and an NbE agent gains $+0.07$ competence on its specialty task. We sweep the decay half-life $H \in \{4, 8, 16\}$ epochs and compute per-task ranks each epoch. For visualization, we focus on one task and plot the rank trajectories of the degraded PbM and the improved NbE agents across time, with a vertical line marking the shock epoch.

\begin{figure}[!htb]
  \centering
  \begin{subfigure}[t]{0.48\textwidth}
    \centering
    \includegraphics[width=\linewidth]{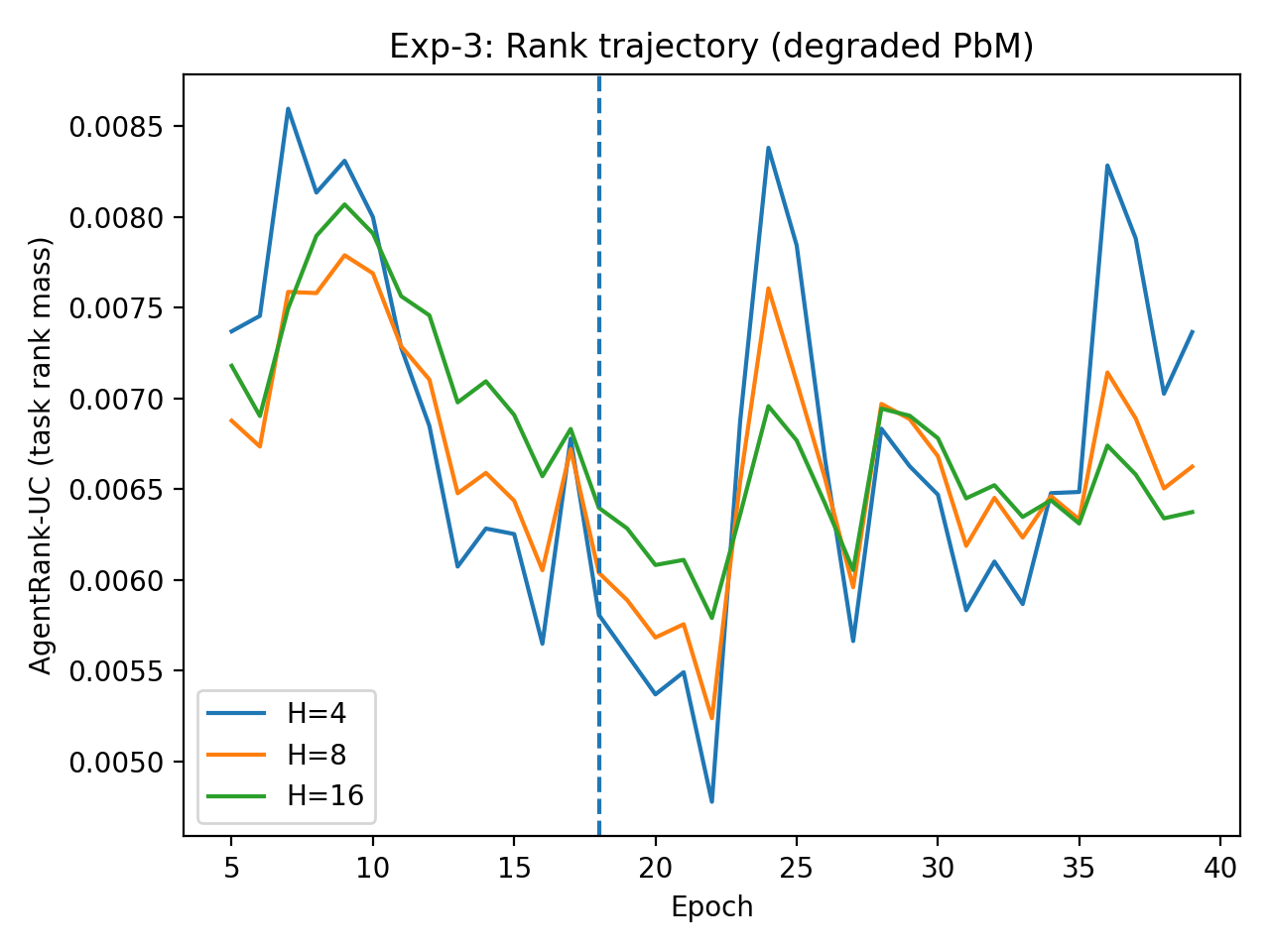}
    \caption{Rank trajectory of the degraded PbM agent for different half-lives $H$. Shorter half-lives demote the degraded agent faster after the shock (vertical dashed line).}
    \label{fig:ranktraj-pbm}
  \end{subfigure}\hfill
  \begin{subfigure}[t]{0.48\textwidth}
    \centering
    \includegraphics[width=\linewidth]{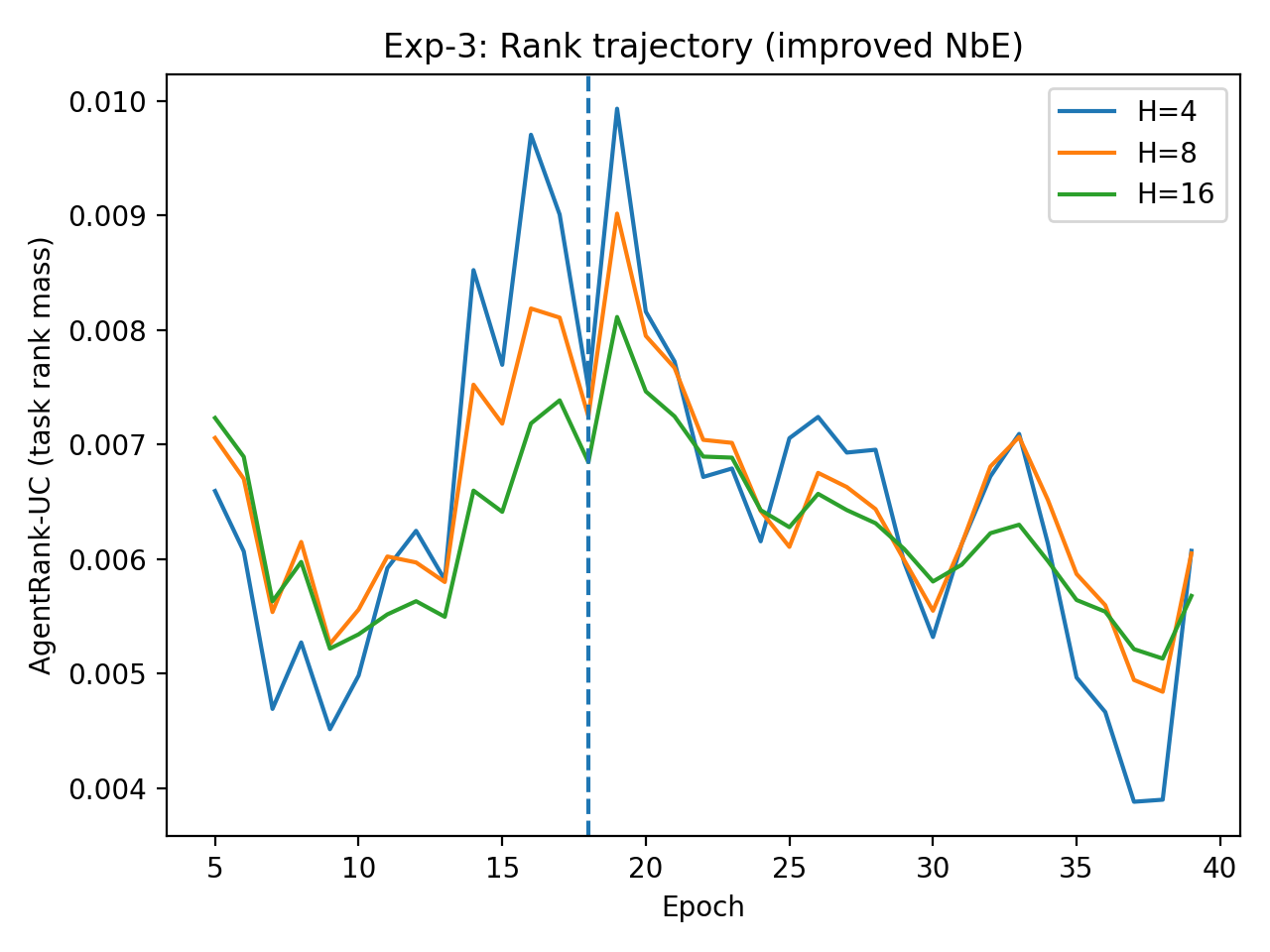}
    \caption{Rank trajectory of the improved NbE agent for different half-lives $H$. Shorter half-lives promote the improving specialist faster after the shock.}
    \label{fig:ranktraj-nbe}
  \end{subfigure}
  \caption{Experiment 3 rank trajectories: degraded PbM agent vs. improved NbE agent under varying half-lives $H$.}
  \label{fig:exp3-ranktrajectories}
\end{figure}

\paragraph{Interpretation.}
The results show that the half-life parameter $H$ provides a clear speed–stability trade-off. Shorter half-lives (e.g., $H{=}4$) respond quickly: the degraded PbM is demoted soon after $t^\star$, while the improving NbE rises rapidly. Longer half-lives (e.g., $H{=}16$) yield smoother curves but slower adaptation. In deployment, $H$ can be tuned to match the expected rate of performance drift: shorter when rapid changes are common, longer when stability is preferred. These dynamics complement the robustness goals of AgentRank-UC by ensuring that rankings reflect recent performance without overreacting to noise.

\subsection{Experiment Four - Monotonicity and Cold-start}

This experiment evaluates two robustness properties of AgentRank-UC: \emph{monotonicity} and \emph{cold-start fairness}. Monotonicity requires that improving an agent’s observed outcomes should never reduce its rank. Cold-start fairness requires that new agents with no history are not suppressed to invisibility, but instead benefit from priors that grant them baseline visibility until evidence accumulates.

We simulate the same world as before with $n=100$ agents and $d=3$ task types. To test monotonicity, we select an agent $j$ and incrementally inject additional successful outcomes into its telemetry after the burn-in phase, while keeping other agents’ outcomes fixed. We track the agent’s rank as a function of the number of injected successes. To test cold-start, we introduce a Newcomer-Good (NC) agent at epoch $t^\star=18$, assigning it high latent competence but no usage history. We compare two prior settings: (i) a uniform prior over all agents, and (ii) a mildly informative prior that gives newcomers a small boost. We then plot the newcomer’s rank trajectory across epochs under the two priors.

\begin{figure}[!htb]
  \centering
  \begin{subfigure}[t]{0.48\textwidth}
    \centering
    \includegraphics[width=\linewidth]{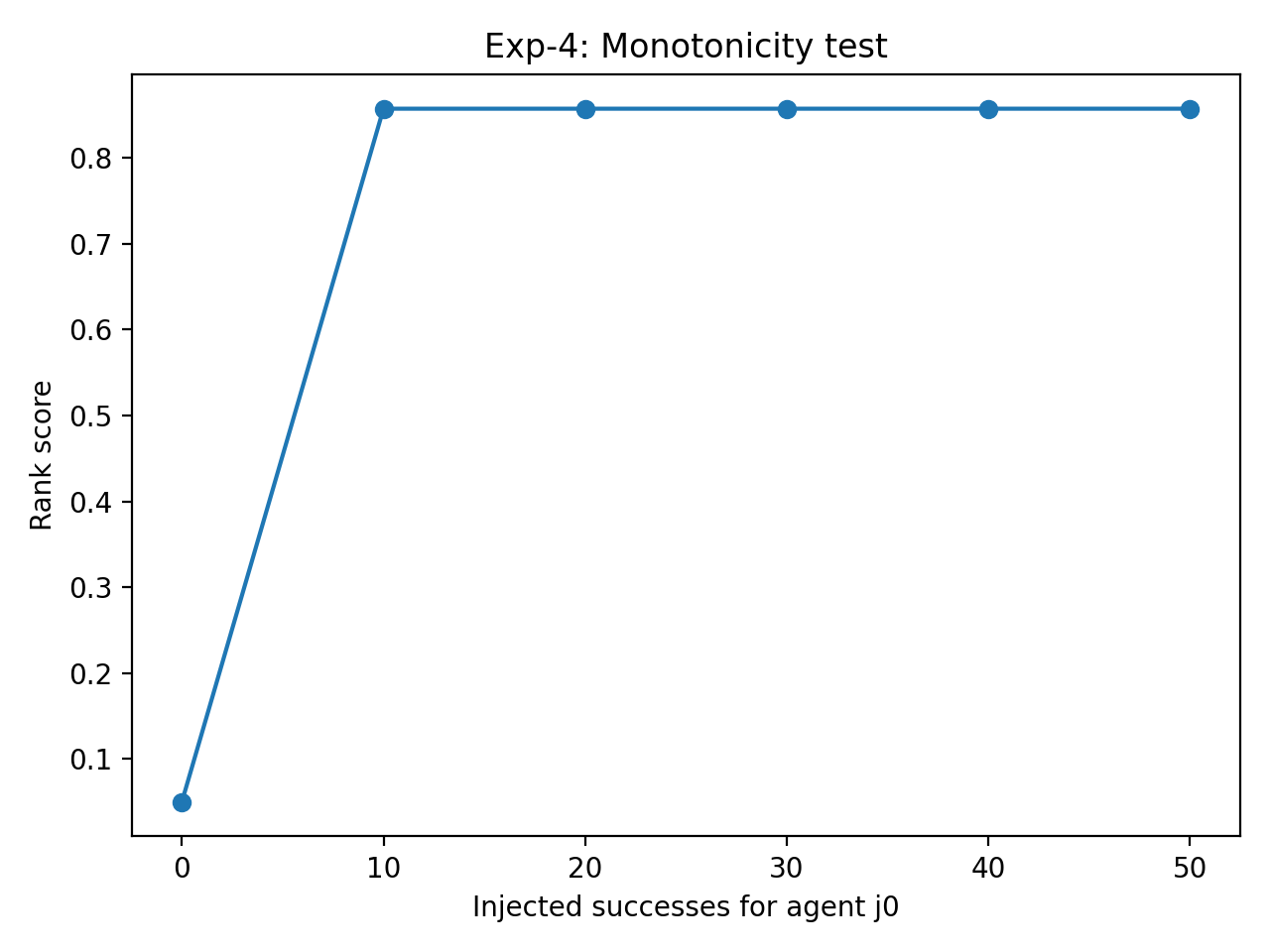}
    \caption{Monotonicity test: the rank of a focal agent increases monotonically as additional successful outcomes are injected, confirming that AgentRank-UC respects monotonicity.}
    \label{fig:monotonicity}
  \end{subfigure}\hfill
  \begin{subfigure}[t]{0.48\textwidth}
    \centering
    \includegraphics[width=\linewidth]{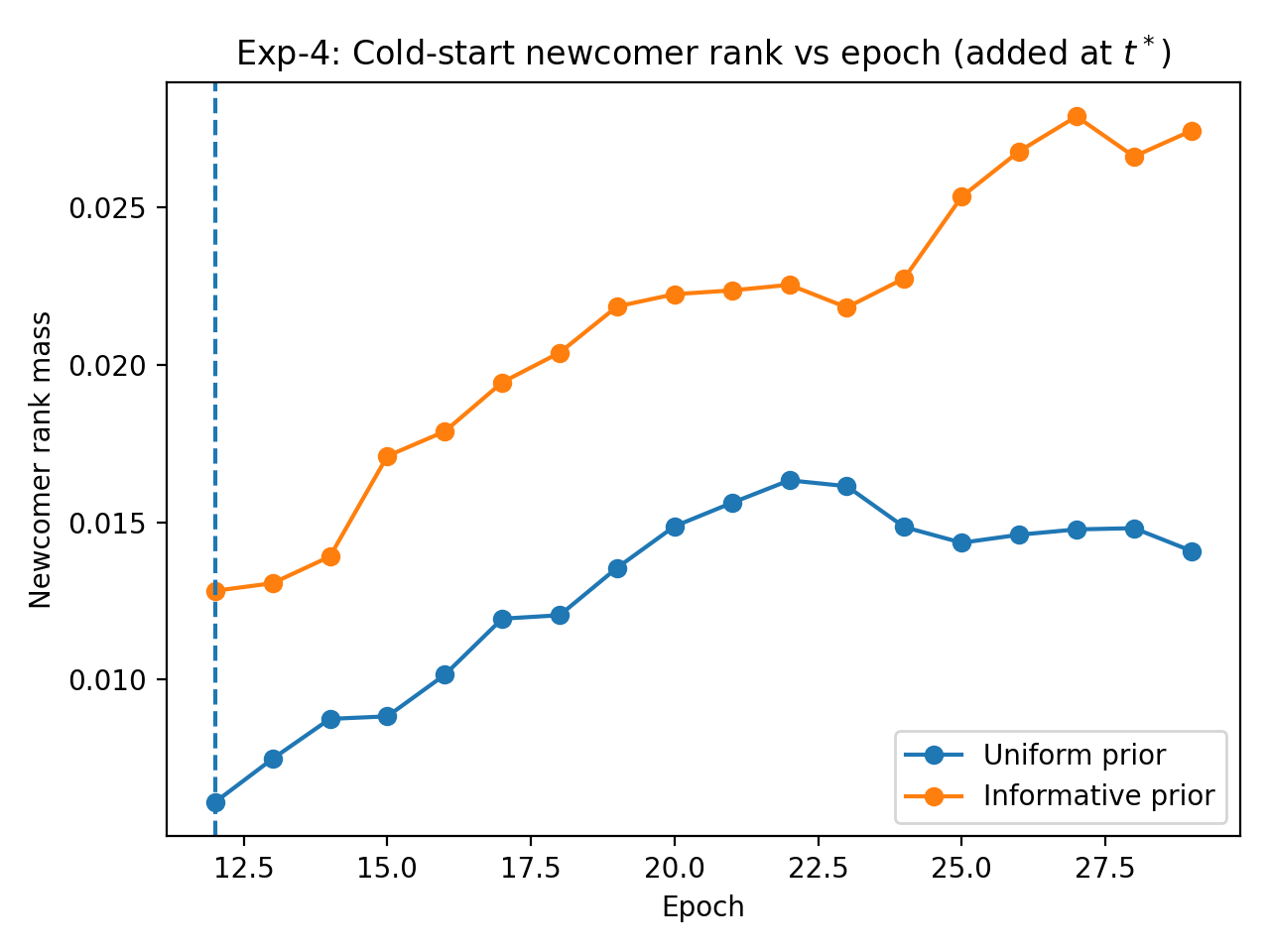}
    \caption{Cold-start test: trajectory of a newcomer agent introduced at epoch $t^\star$. With uniform priors, the agent gains visibility gradually; with informative priors, it enters rankings immediately and rises faster.}
    \label{fig:coldstart}
  \end{subfigure}
  \caption{Experiment 4 diagnostic tests: (a) monotonicity and (b) cold-start behavior.}
  \label{fig:exp4-tests}
\end{figure}

\paragraph{Interpretation.}
The monotonicity test shows that AgentRank-UC ranks increase strictly as more successes are added, demonstrating that the algorithm never penalizes improvements. The cold-start test confirms that priors prevent newcomers from being invisible: with uniform priors, the newcomer rises as evidence accumulates; with informative priors, it gains earlier visibility and converges more quickly to its competence level. These results highlight that AgentRank-UC guarantees fairness both for improving incumbents and for high-quality entrants.

\subsection{Experiment Five- Sybil resistance }

This experiment evaluates robustness to usage pumping by collusive Sybil cliques. In usage-based discovery, colluding agents can inflate their apparent importance by calling one another frequently, even if their true competence is mediocre. AgentRank-UC should reduce the rank mass assigned to such cliques by relying on competence-weighted signals and multi-objective utility, whereas usage-only baselines are vulnerable to inflated call volumes.

We instantiate the world from Section~4.2 with a Sybil clique (5--10 agents) that preferentially call one another. Sybil callees are also slightly slower and riskier than average (small penalties in the utility). The simulation runs for 36 epochs with a 5-epoch burn-in; per-task ranks are recomputed each epoch and fed into caller routing. At the final epoch we compute per-task rankings for four methods---AgentRank-UC (UC), usage-only (Usage), competence-only (Comp), and a naive success-rate oracle (Oracle)---and evaluate (i) \emph{SybilMass}, the total rank mass assigned to the Sybil clique, and (ii) \emph{Quality@10 excluding Sybil}, the average true competence of the top-10 non-Sybil agents. We also plot SybilMass over time for UC and Usage (averaged across tasks).

\begin{figure}[!htb]
  \centering
  \begin{subfigure}[t]{0.48\textwidth}
    \centering
    \includegraphics[width=\linewidth]{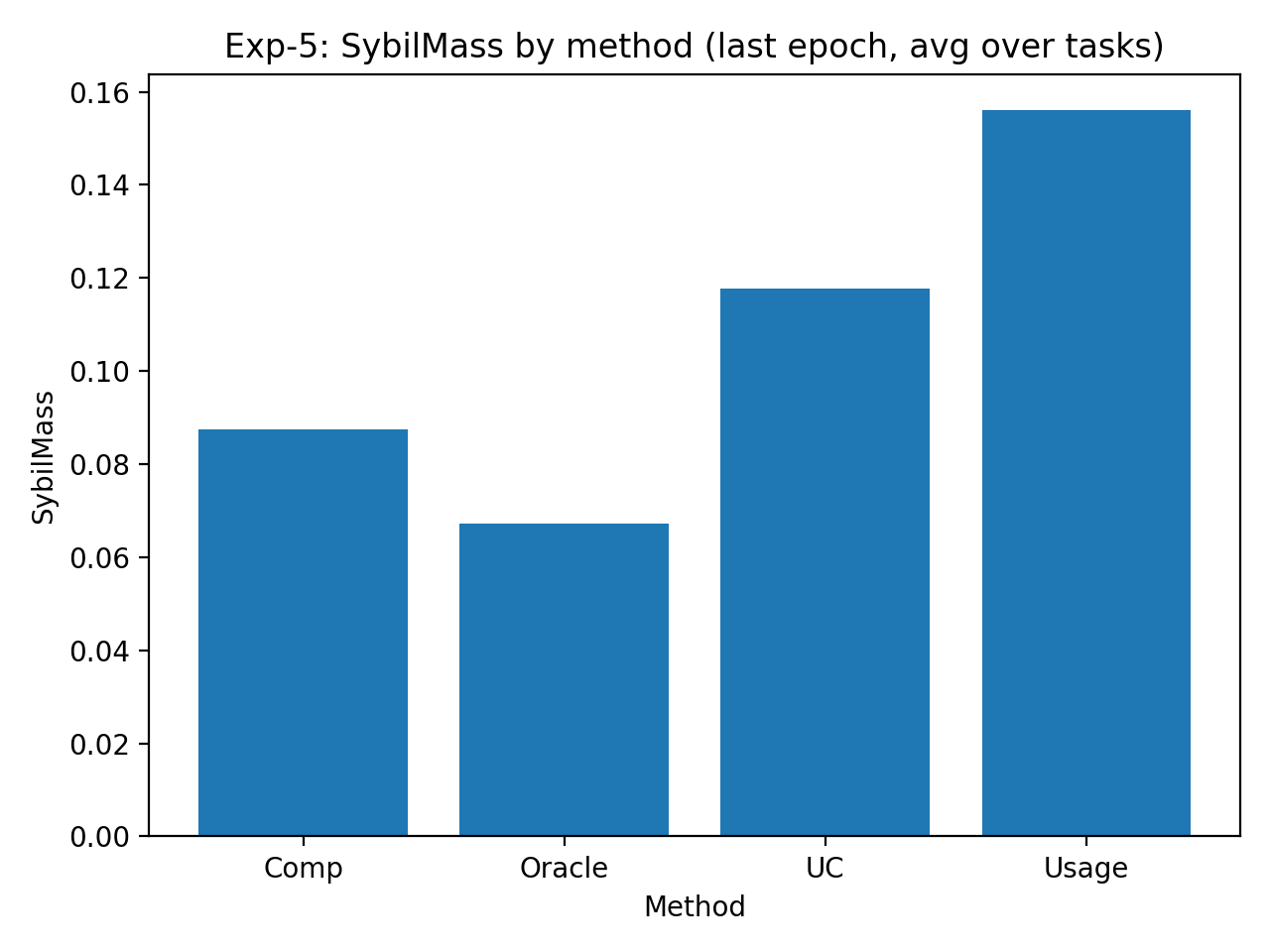}
    \caption{SybilMass by method (last epoch, averaged over tasks). Usage-only assigns the highest mass to the Sybil clique; AgentRank-UC significantly reduces Sybil mass relative to usage-only.}
    \label{fig:sybil-bar}
  \end{subfigure}\hfill
  \begin{subfigure}[t]{0.48\textwidth}
    \centering
    \includegraphics[width=\linewidth]{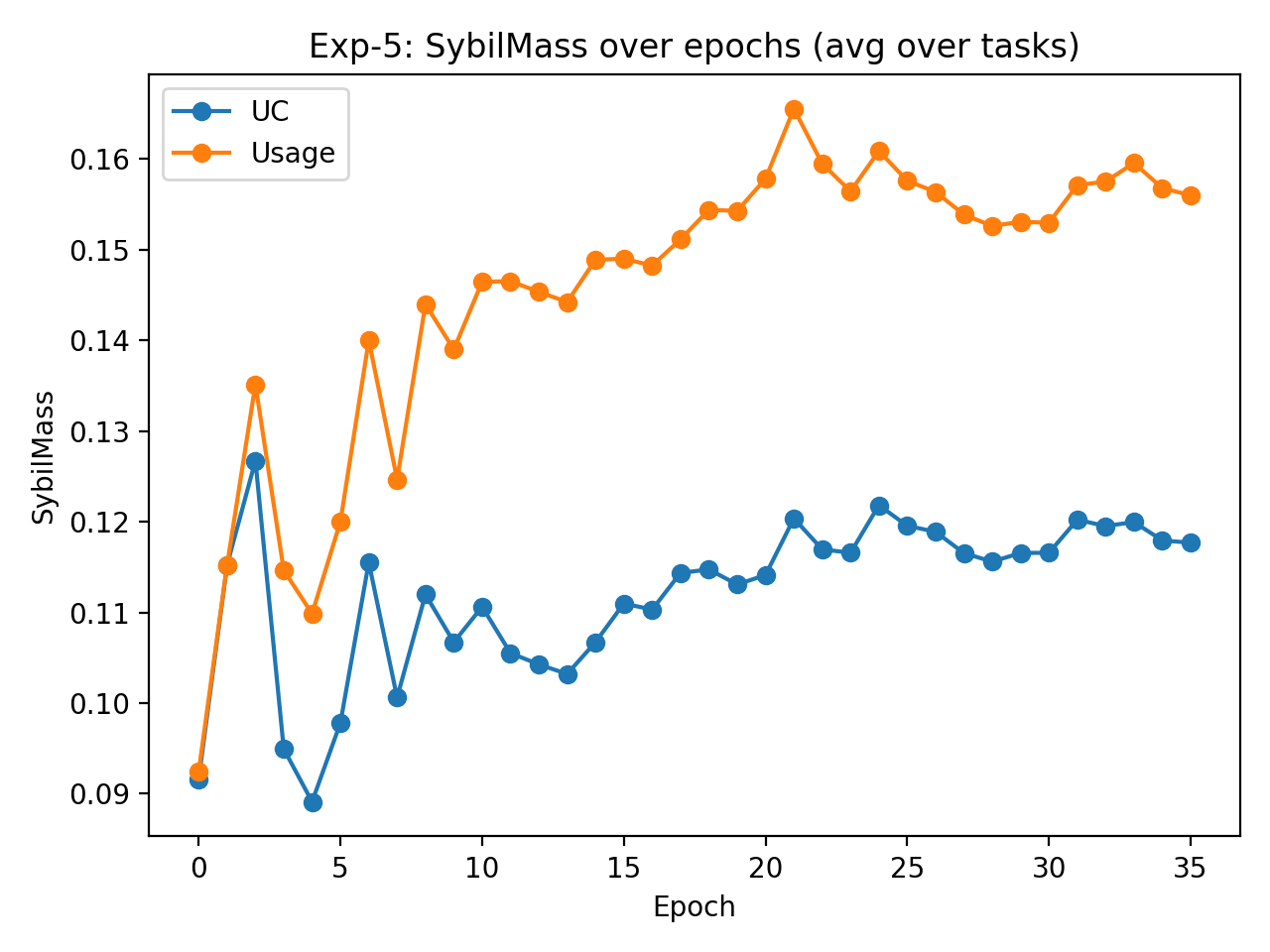}
    \caption{SybilMass over epochs (averaged over tasks). Under usage-only, the Sybil mass grows; under AgentRank-UC, it is contained and, after burn-in, declines.}
    \label{fig:sybil-time}
  \end{subfigure}
  \caption{Experiment 5: SybilMass comparison. (a) Sybil mass by method at final epoch. (b) Sybil mass evolution over time.}
  \label{fig:exp5-sybilmass}
\end{figure}


\begin{table}[h]
\centering
\caption{Experiment 5 : Sybil Mass and Quality@10 (excluding Sybils) by task and method. Values are rounded to two decimals.}
\begin{tabularx}{\linewidth}{l l S[table-format=1.2] S[table-format=1.2]}
\toprule
\textbf{Task} & \textbf{Method} & \textbf{SybilMass} & \textbf{Quality@10\_exclSY} \\
\midrule
0 & UC     & 0.11 & 0.85 \\
0 & Usage  & 0.14 & 0.83 \\
0 & Comp   & 0.08 & 0.86 \\
0 & Oracle & 0.07 & 0.88 \\
\midrule
1 & UC     & 0.12 & 0.81 \\
1 & Usage  & 0.16 & 0.77 \\
1 & Comp   & 0.09 & 0.84 \\
1 & Oracle & 0.07 & 0.88 \\
\midrule
2 & UC     & 0.12 & 0.83 \\
2 & Usage  & 0.17 & 0.79 \\
2 & Comp   & 0.08 & 0.82 \\
2 & Oracle & 0.06 & 0.87 \\
\bottomrule
\end{tabularx}
\label{tab:sybil-metrics}
\end{table}

\paragraph{Interpretation.}
Usage-only allocates the most rank mass to the Sybil clique, reflecting its vulnerability to usage pumping. AgentRank-UC assigns substantially less mass to Sybils and maintains higher Quality@10 among non-Sybil agents, due to competence-weighted edges and penalties for latency/cost/risk. The oracle and competence-only baselines remain relatively unaffected in this caller-independent world, reinforcing that usage signals are the primary attack surface. These results confirm that AgentRank-UC mitigates collusive usage inflation while preserving discovery quality.

\subsection{Discussion}

Taken together, the five experiments provide a comprehensive picture of how AgentRank-UC behaves under a variety of conditions.  

\begin{itemize}
    \item \textbf{Sanity check.} Experiment~1 confirms that the algorithm is well-defined and produces sensible rankings aligned with true competence, establishing a foundation for subsequent tests.  
    \item \textbf{Interpolation control.} Experiment~2 shows that the merge parameter $p$ acts as a smooth and interpretable knob: with frozen telemetry, competence-only and usage-only form flat baselines, and UC interpolates predictably between them.  
    \item \textbf{Adaptivity.} Experiment~3 demonstrates that with exponential decay, rankings adapt promptly to shocks: degraded agents lose rank, while improved agents gain visibility at a speed determined by the half-life parameter.  
    \item \textbf{Fairness.} Experiment~4 verifies that rankings are monotone in outcomes (more successes never hurt) and that newcomers retain visibility through priors, avoiding cold-start suppression.  
    \item \textbf{Robustness.} Experiment~5 shows that collusive Sybil cliques can inflate usage-only rankings but have limited influence under UC, which downweights popularity not matched by competence.  
\end{itemize}

These findings highlight complementary strengths: UC provides interpretability (via $p$), responsiveness (via decay), fairness (via monotonicity and priors), and robustness (via competence weighting). Together, they offer empirical support that the DOVIS telemetry, combined with AgentRank-UC, yields rankings that are not only well-founded in theory but also practical and resilient in realistic environments.

\section{Future Directions}

This paper has introduced DOVIS, a minimal protocol for telemetry in open agent ecosystems, and AgentRank-UC, a usage--competence ranking algorithm with theoretical guarantees and empirical validation. While these contributions establish a foundation, many opportunities remain for refinement and extension. We outline several promising future directions.

\subsection{Extending the DOVIS protocol}

The current OAT-Lite schema was deliberately minimal: call counts, success rates, and aggregated quality, latency, cost, and risk. A natural extension is an \emph{OAT-Full} schema capturing richer signals such as energy efficiency, fairness, calibration, and interpretability. Privacy-preserving telemetry is another priority: secure aggregation, trusted execution environments, and differential privacy mechanisms could allow agents to contribute useful signals without exposing sensitive data. Incentive design also warrants deeper study: tokenized credits, stake-slashing for misreporting, or exposure multipliers for verified telemetry may align reporting incentives more effectively. Finally, federated adoption across multiple marketplaces suggests the need for global task registries and interoperable schemas to enable cross-market agent discovery \cite{OpenTelemetry,Bonawitz2017,DworkRoth2014,Abadi2016,Costan2016}.

\subsection{Refining AgentRank-UC}

On the algorithmic side, many refinements are possible. Rather than fixing the blend parameter $p$, one could adapt $p$ dynamically as a function of task type or traffic sparsity, or even learn it end-to-end. Caller-specific teleport vectors $(v,w)$ could yield personalized discovery views. Beyond the geometric mean, alternative fusion operators (e.g., weighted harmonic means or divergences inspired by information theory) may offer stronger fairness or robustness while preserving monotonicity. On the theoretical front, there is scope for sharper perturbation bounds, spectral gap analyses of convergence rates, and fairness guarantees for underrepresented agents under skewed priors. At the same time, impossibility results---showing where no ranking can resist certain adversarial strategies---would clarify the limits of competence-aware discovery \cite{LangvilleMeyer2006,Gleich2015}.

\subsection{Broader research agenda}

Beyond protocol and algorithmic refinements, several broader directions suggest themselves. Adversarial robustness remains an open challenge: formalizing attack models such as collusion, Sybil amplification, and data poisoning, and proving lower bounds or resistance guarantees. Privacy-preserving ranking is another: combining DOVIS telemetry with differential privacy, multi-party computation, or federated learning to prevent leakage of sensitive usage patterns. A natural evolution is federated or decentralized ranking, in which multiple marketplaces collectively maintain discovery scores through distributed consensus, extending PageRank’s web-scale intuition to the agentic web. More broadly, AgentRank-UC could be integrated into bandit or reinforcement learning frameworks as a feature to optimize task routing. Finally, the socio-technical implications of agent ranking merit study: how different ranking choices shape economic incentives, fairness, and governance in ecosystems of autonomous agents.

\medskip
In summary, while DOVIS and AgentRank-UC provide a principled starting point, they open a wide landscape of technical, theoretical, and societal research questions. Addressing these will be crucial for building robust, fair, and privacy-preserving discovery mechanisms in the emerging agentic web \cite{yang2025agentic,Park2023,Li2023,Wu2023}.

\section{Conclusion}

We have introduced DOVIS, a minimal protocol for telemetry in open agent ecosystems, and AgentRank-UC, a usage--competence ranking algorithm that combines popularity and performance into a single principled discovery score. DOVIS provides the first concrete specification of how agents can publish minimal, privacy-preserving telemetry at scale, while AgentRank-UC demonstrates how such signals can be aggregated into a ranking with formal guarantees. Our theoretical analysis established existence, uniqueness, monotonicity, cold-start fairness, stability under recency decay, and resistance to Sybil amplification. Complementary simulations illustrated how AgentRank-UC achieves near-oracle performance in stationary settings, adapts quickly to shocks, fairly incorporates newcomers, and resists usage pumping by collusive agents.

Taken together, these contributions mark a step toward a competence-aware discovery substrate for the emerging agentic web. By providing both a protocol (DOVIS) and an algorithm (AgentRank-UC), along with theory and simulations, we offer a framework that is implementable today yet extensible for future research. Many directions remain open, from richer telemetry and stronger incentives to adaptive fusion, privacy-preserving ranking, and decentralized deployment. We hope that DOVIS and AgentRank-UC can serve as a foundation for a broader research program on trustworthy discovery in agent ecosystems, and ultimately contribute to building a more robust, fair, and transparent agentic web.


\bibliographystyle{unsrt}
\bibliography{references}

\end{document}